\documentclass[journal, 10pt]{IEEEtran}
\usepackage{cite}
\usepackage{url}
\usepackage{amssymb}
\usepackage{amsmath}
\usepackage{amsthm}
\usepackage{algorithmic}
\usepackage{graphicx}
\usepackage{multirow}
\usepackage{color}
\newtheorem{theorem}{Theorem}

\newtheorem{lemma}{Lemma}

\newtheorem*{assumption*}{Assumption}

\newtheorem{example-non*}{Example}

\newtheorem{remark}{Remark}
\allowdisplaybreaks 
\usepackage{enumerate}

\DeclareMathOperator*{\argmin}{arg\,min}
\DeclareMathOperator*{\argmax}{arg\,max}

\newcommand{\comment}[1]{}

\ifodd 0
\newcommand{\newc}[1]{{\color{blue}#1}} 
\else
\newcommand{\newc}[1]{#1}
\fi

\ifodd 0
\newcommand{\jremove}[1]{}
\newcommand{\aremove}[1]{{\color{black}#1}} 
\else
\newcommand{\jremove}[1]{{\color{black}#1}} 
\newcommand{\aremove}[1]{}
\fi

\begin{document}
\title{RELEAF: An Algorithm for Learning and Exploiting Relevance}
\author{\IEEEauthorblockN{Cem Tekin,~\IEEEmembership{Member,~IEEE}, Mihaela van der Schaar,~\IEEEmembership{Fellow,~IEEE}\\}
\thanks{Copyright (c) 2015 IEEE. Personal use of this material is permitted. However, permission to use this material for any other purposes must be obtained from the IEEE by sending a request to pubs-permissions@ieee.org.}
\thanks{This work is partially supported by the grants NSF CNS 1016081 and
AFOSR DDDAS.}
\thanks{C. Tekin and Mihaela van der Schaar are in Department of Electrical Engineering, UCLA, Los Angeles, CA, 90095. Email: cmtkn@ucla.edu, mihaela@ee.ucla.edu.}
\thanks{A preliminary version of this paper appeared in NIPS 2014 \cite{tekinNIPS2014}.}
\jremove{This online appendix is an extended version of our paper accepted to IEEE JSTSP \cite{tekinJSTSP2015}.}
}

\maketitle
%

\begin{abstract}
Recommender systems, medical diagnosis, network security, etc., require on-going learning and decision-making in real time.  These -- and many others -- represent perfect examples of the opportunities and difficulties presented by Big Data: the available information often arrives from a variety of sources and has diverse features so that learning from all the sources may be valuable but integrating what is learned is subject to the curse of dimensionality.  This paper develops and analyzes algorithms that allow efficient learning and decision-making while avoiding the curse of dimensionality.  We formalize the information available to the learner/decision-maker at a particular time as a {\em context vector} which the learner should consider when taking actions.  In general the context vector is very high dimensional, but in many settings, the most relevant information is embedded into only a few {\em relevant} dimensions.  If these relevant dimensions were known in advance, the problem would be simple -- but they are not.  Moreover, the relevant dimensions may be different for different actions.  Our algorithm {\em learns} the relevant dimensions for each action, and makes decisions based in what it has learned. Formally, we build on the structure of a contextual multi-armed bandit by adding and exploiting a {\em relevance relation}.  \newc{We prove a general regret bound for our algorithm whose time order depends only on the maximum number of relevant dimensions among all the actions, which in the special case where the relevance relation is single-valued (a function), reduces to $\tilde{O}(T^{2(\sqrt{2}-1)})$}; in the absence of a relevance relation, the best known contextual bandit algorithms achieve regret $\tilde{O}(T^{(D+1)/(D+2)})$, where $D$ is the full dimension of the context vector.   Our algorithm alternates between exploring and exploiting and does not require observing outcomes during exploitation (so allows for active learning).  Moreover, during exploitation, suboptimal actions are chosen with arbitrarily low probability.  Our algorithm is tested on datasets arising from breast cancer diagnosis, network security and online news article recommendations.
\end{abstract}

\vspace{-0.1in}
\section{Introduction}\label{sec:intro}

The world is increasingly information-driven. Vast amounts of data are being produced by diverse
sources and in diverse formats including sensor readings, physiological measurements, documents,
emails, transactions, tweets, and audio or video files and many businesses and government institutions
rely on these Big Data in their everyday operations.  (Particular applications that have been discussed in the literature include recommender systems \cite{li2010contextual},  neuroscience \cite{djolonga2013high}, network monitoring \cite{gao2007appropriate}, surveillance \cite{stauffer2000learning}, health monitoring \cite{tseng2008integrated}, stock market prediction, intelligent driver assistance \cite{avidan2004support}, etc.) To make the best use of these data, it is vital to learn from and respond to the streams of data continuously and in real time.  Because data streams are heterogeneous and dynamically evolving over time in unknown and unpredictable ways, making decisions using these data streams online, at run-time, is known to be a very challenging problem \cite{ducasse2010adaptive, IBMsmarter}.  In this paper, we tackle these online Big Data challenges by exploiting a feature that is common to many applications: the data may have many dimensions, but the information that is most important for any given action is embedded into only a few {\em relevant} dimensions.  In general, these relevant dimensions will be different for different actions and are not known in advance -- so must be {\em learned}.  We propose and analyze an algorithm that {\em learns} the relevant dimensions for each action, and makes decisions based in what it has learned.  

Our structure builds on contextual multi-armed bandits. We formalize the information obtained from the data streams (perhaps after pre-processing) in terms of ``context vectors".  Context vectors characterize the information contained in the data generated by the process the learner wishes to control/act on such as the location, and/or data type information (e.g., features/characteristics/modality).  The decision maker/learner receives the context vector and takes an action that generates a reward that depends (stochastically) on the context vector.  Contexts, actions and rewards are generic terms; the specific meaning depends on the specific Big Data application. For instance, in a network security application \cite{gao2007appropriate}, contexts are the features of the network packet, actions are the set of predictions about the type of network attacks and the reward is the accuracy of the prediction. In a recommender system \cite{li2010contextual}, contexts are the characteristics (age, gender, purchase history, etc.) of the user, actions are items and the reward is the indicator function of the event that the user buys the item. 
The problem is to {\em learn} the rewards (or the distribution of rewards) generated by each action in each context.  The context vector is typically high dimensional but in many applications the reward for a particular action will depend only on a few most relevant of these dimensions, embodied in a {\em relevance relation}. For an action set ${\cal A}$ and a type (dimension) set ${\cal D}$, the relevance relation is given by $\boldsymbol{{\cal R}} = \{ {\cal R}(a)\}_{a \in {\cal A}}$, where ${\cal R}(a) \subset {\cal D}$.  However, whether this is the case and if so, which dimensions are most relevant for a particular action, is not known in advance but must be learned, and decision-making must be adapted to this learning process.

Relevance relations arise naturally in many practical applications. For example, when treating patients with a particular disease, many contexts may be available -- the patients' age, weight, blood tests, imaging, medical history etc. –- but often only a few of these contexts are relevant in choosing/not choosing a particular treatment or medication.   For instance, surgery may be strongly contra-indicated in patients with clotting problems; drug therapies that require close monitoring may be strongly contra-indicated in patients who do not have committed care-givers, etc.  Similarly, in recommender systems, a product recommendation may sometimes depend on many characteristics of the user -- gender, occupation, history of past purchases etc. -– but will often depend only (or most strongly) on a few characteristics -- such as location and home-ownership. 

Relevance allows us to avoid the curse of dimensionality: we show that regret bounds depend only on the number of {\em relevant dimensions}, i.e., $D_{\textrm{rel}}$ -- which is typically much less than the full number of dimensions.
Our main contributions can be summarized as follows:

\begin{itemize}
\item We propose the {\em Relevance Learning with Feedback} (RELEAF) algorithm that alternates between exploration and exploitation phases.  
\newc{For the general case when $D_{\textrm{rel}}<D/2$, RELEAF achieves a regret bound of $\tilde{O}(T^{g(D_{\textrm{rel}})})$,\footnote{$O(\cdot)$ is the Big O notation, $\tilde{O}(\cdot)$  is the same as $O(\cdot)$  except it hides terms that have polylogarithmic growth.} where $g(D_{\textrm{rel}}) \leq (2D_{\textrm{rel}}+3)/(2D_{\textrm{rel}}+4)$, which reduces to a regret bound of $\tilde{O}(T^{2(\sqrt{2}-1)})$ when the relevance relation is a function.}
\item We derive separate bounds on the regret incurred in exploration and exploitation phases. RELEAF only needs to observe the reward in exploration phases and hence, when observing rewards is costly, active learning can be performed by controlling reward feedback. RELEAF achieves the same time order of regret even when observing rewards is costly.
\item The operation of RELEAF involves a confidence parameter, chosen by the user, which can be arbitrarily small.  If confidence $\delta$ is chosen, then RELEAF will never select suboptimal actions in exploitation steps with probability at least $1-\delta$.  This provides performance guarantees, which are important -- perhaps vital -- in many applications, such as medical treatment. 
\end{itemize}


The rest of the paper is organized as follows. Related work is given in Section \ref{sec:related}. The problem is formalized in Section \ref{sec:probform}. An algorithm that learns the {\em relevance relation} between actions and types of contexts is given in Section \ref{sec:algodescription}. Then, the regret bounds are proved for this algorithm. Numerical results on several real-world datasets are given in Section \ref{sec:numerical}. Finally, conclusions are given in Section \ref{sec:conc}.

\vspace{-0.1in}
\section{Related Work} \label{sec:related}

\comment{
Related work is categorized and discussed in the following subsections. Table \ref{tab:relatedwork} summarizes the differences between related work and our work. 

\begin{table*}
\centering
{\fontsize{9}{7}\selectfont
\setlength{\tabcolsep}{.1em}
\begin{tabular}{|c|c|c|c|c|}
\hline
Model  & Assumption &  Learning $\boldsymbol{{\cal R}}$ & Regret bound & Reward Feedback \\
\hline
Contextual bandits & Similarity metric (Lipschitz continuity) & No  & Worst-case $\tilde{O}(T^{(D+1)/(D+2)})$ & Always  \\
\hline
Dimensionality reduction methods &  &  error & &  \\
\hline
Label efficient learning &  &  error & &  \\
\hline
Active learning &  &  error & &  \\
\hline
\end{tabular}
}
\caption{Comparison of the error rates of RELEAF-FO with ensemble learning methods for network intrusion dataset.}
\label{tab:schemeintrusion}
\end{table*}
}

\subsection{Multi-armed bandits}

Our work is a new contextual bandit problem where relevance relations exist. 
Contextual bandit problems are studied by many others in the past \cite{hazan2007online, langford2007epoch, lu2010contextual, slivkins2011contextual, chu2011contextual, dudik2011efficient}. The problem we consider in this paper is a special case of the Lipschitz contextual bandit problem \cite{lu2010contextual, slivkins2011contextual}, where the only assumption is the existence of a known similarity metric between the expected rewards of actions for different contexts. \newc{The strengh of this model comes from the fact that there are no stochastic assumptions made on the context arrival process, and the benchmark which the regret is defined against selects the best action for each context.}
It is known that the lower bound on regret for this problem is $O(T^{(D+1)/(D+2)})$ \cite{lu2010contextual}, and there exists algorithms that achieve $\tilde{O}(T^{(D+1)/(D+2)})$ regret \cite{lu2010contextual, slivkins2011contextual}.\footnote{The bounds in \cite{lu2010contextual, slivkins2011contextual} are given in terms of {\em covering} and {\em zooming} dimensions of the problem instance, but they reduce to the Euclidian dimension for the set of assumptions we have in this paper.} Compared to these works, RELEAF only needs to observe rewards in explorations and has a regret whose time order is independent of $D$. Hence it can still learn the optimal actions fast enough in settings where observations are costly and the context vector is high dimensional. 
\newc{For instance, in Section \ref{sec:regretRELEAF2} we show that the regret of RELEAF is better than the bound of $\tilde{O}(T^{(D+1)/(D+2)})$ in \cite{lu2010contextual, slivkins2011contextual} for $D_{\textrm{rel}} \leq D/2-1$.}

\newc{
Another class of contextual bandit problems consider reward functions that are linear in the contexts \cite{li2010contextual,chu2011contextual}. Due to this linearity assumption learning reduces to estimating the parameter vector corresponding to each arm, hence the regret bounds do not depend on the dimension of the context space.}
\newc{
Several papers \cite{langford2007epoch,dudik2011efficient} impose stochastic assumptions on the process that generates the contexts and the arm rewards. For instance assuming that the contexts and arm rewards are generated by an unknown i.i.d. process, regret independent of the dimension of the context space can be achieved. 

The differences between our work and these prior works are summarized in Table \ref{tab:banditcompare}.
}

\begin{table}
\centering
{
{\fontsize{9}{9}\selectfont
\setlength{\tabcolsep}{.1em}
\begin{tabular}{|c|c|c|c|c|}
\hline
 & Our work & \cite{lu2010contextual, slivkins2011contextual} & \cite{li2010contextual,chu2011contextual} & \cite{langford2007epoch,dudik2011efficient}  \\
\hline
Relevance relation & yes & no & no & no \\
\hline
Context arrivals & arbitrary & arbitrary & arbitrary & i.i.d.  \\
\hline
Reward-context & Lipschitz  & Lipschitz  & linear & joint i.i.d.  \\
relation & & & & process \\
\hline
Time order of  & depends   & depends  & independent  & independent   \\
the regret & on $D_{\textrm{rel}}$  & on  $D$ & of $D$ & of $D$ \\
\hline
\end{tabular}
}
}
\caption{Comparison of our work with other work contextual bandits.}
\vspace{-0.3in}
\label{tab:banditcompare}
\end{table}

\vspace{-0.1in}
\subsection{Dimensionality reduction}

Dimensionality reduction methods are often used to find low dimensional representations of high dimensional context vectors (feature vectors) such that the information contained in the low dimensional representation is approximately equal to the information contained in the original context vector \cite{srivastava2005tools}. 
For instance, reduced-rank adaptive filtering \cite{haimovich1991eigenanalysis, hua2001optimal, de2009adaptive} first projects feature vectors onto a lower dimensional subspace, and then adaptively adjusts the filter coefficients over time. 
In these works a low dimensional representation of the feature vector is learned based on the available data. 
Compared to this, in our work the relevant dimensions for each action can be different, hence a low dimensional representation that contains information about the rewards of all actions may not exist. An example is a relevance relation $\boldsymbol{{\cal R}}$ for which each action only has few relevant dimensions, i.e., $D_{\textrm{rel}} << D$, but $\bigcup_{a \in {\cal A}} {\cal R}(a) =D$.



\vspace{-0.1in}
\subsection{Learning with limited number of observations}
Examples of related works that consider limited observations while learning are KWIK learning \cite{li2011knows, amin2013large} and label efficient learning \cite{cesa2009robust, kakade2008efficient, hazan2011newtron}. 
For example, \cite{amin2013large} considers a bandit model where the reward function comes from a parameterized family of functions and gives a bound on the average regret. 
An online prediction problem is considered in \cite{cesa2009robust, kakade2008efficient, hazan2011newtron}, where the predictor (action) lies in a class of linear predictors. The benchmark of the context is the best linear predictor. This restriction plays a crucial role in deriving regret bounds whose time order does not depend on $D$. 
Similar to these works, RELEAF can guarantee with a high probability that actions with suboptimality greater than a desired $\epsilon>0$ will never be selected in exploitation steps. 
However, we do not have any assumptions on the form of the expected reward function other than the Lipschitz continuity. 

For the special case when actions correspond to making predictions about the context vector (which is equal to the data stream for this special case), our problem is closely related to the problem of active learning.
In this problem, obtaining the labels is costly, but the performance of the learning algorithm, i.e., rewards, can only be assessed through the labels, hence actively learning when
to ask for the label becomes an important challenge.
In stream-based active learning \cite{freund1997selective, cesa2005minimizing, dekel2012selective, zolghadr2013online}, the learner is provided with a stream of unlabeled instances. 
When an instance arrives, the learner decides to obtain the label or not. 
To the best of our knowledge there is no prior work in stream-based active learning that deals with learning relevance relations with sublinear bounds on the regret. 

\vspace{-0.1in}
\subsection{Ensemble learning}

Numerous ensemble learning methods exists in the literature \cite{gao2007appropriate, freund1995desicion, fan1999application, blum1997empirical}. These methods take predictions (actions) from a set of experts (e.g., base classifiers), and combine them with a specific rule to produce a final prediction (action). After the reward of all the actions are observed, the rule to combine the predictions of the experts is updated based on how good each individual expert had performed. 
The goal is to learn a combination rule such that even if the predictions' of the individual experts are not very accurate, the final prediction is accurate because it takes into account the ``opinions" of all experts. 

To evaluate the performance of ensemble learning methods analytically, the benchmark is usually taken to be the expert that achieves the highest total reward. Hence the ``quality'' of the regret bounds depends on the ``quality'' of the experts.
In contrast, our regret bounds are with respect to the best benchmark (that only depends on context arrivals and reward distributions), and can be applied to settings without experts. 
Moreover, our algorithms work for the {\em bandit setting}, in which after an action is chosen, only its reward is revealed to the algorithm.




\section{\newc{Problem Formulation and Preliminaries}}\label{sec:probform}

\subsection{Notation}

\newc{
For a vector $\boldsymbol{x}$, $x_i$ denotes its $i$th component. Given a vector $\boldsymbol{v}$, $\boldsymbol{x}_{\boldsymbol{v}} := \{ x_i  \}_{i \in \boldsymbol{v}}$ denotes the components of $\boldsymbol{x}$ whose positions are in  $\boldsymbol{v}$. The time index is $t=1,2,\ldots$. When referring to a time dependent variable we use subscript $t$ as the rightmost subscript corresponding to that variable. For instance $\boldsymbol{x}_t$ denotes a vector at time $t$, $x_{i,t}$ denotes its $i$th component at time $t$, and $\boldsymbol{x}_{\boldsymbol{v},t}$ denotes the vector of its components that are in $\boldsymbol{v}$ at time $t$. 
}

\subsection{Problem formulation}

${\cal A}$ is the set of actions, $D$ is the dimension of the context vector, ${\cal D} := \{1,2,\ldots,D\}$ is the set of types, and $\boldsymbol{{\cal R}} = \{ {\cal R}(a) \}_{a \in {\cal A}}: {\cal A} \rightarrow 2^{\cal D}$ is the \newc{(unknown)} relevance relation, which maps every $a \in {\cal A}$ to a subset of ${\cal D}$. 
We call $D_{\textrm{rel}} = \max_{a \in {\cal A}} |{\cal R}(a)|$, the {\em relevance dimension}. When 
 $D_{\textrm{rel}} =1$, we say that $\boldsymbol{{\cal R}}$ is a {\em relevance function}. Elements of ${\cal D}$ are denoted by index $i$. 
 Let ${\cal V}_K$, $1\leq K \leq D$ be the set of $K$ element subsets of ${\cal D}$. We call $\boldsymbol{v} \in {\cal V}_K$,  a $K$-tuple of types.

At each time step $t=1,2,\ldots$, a context vector $\boldsymbol{x}_{t}$
arrives to the learner. After observing $\boldsymbol{x}_{t}$ the
learner selects an action $a\in{\cal A}$, which results in a random
reward $r_{t}(a,\boldsymbol{x}_{t})$. The learner may choose to observe
this reward by paying cost $c_{O} \geq 0$. The goal of the learner is
to maximize the sum of the generated rewards minus costs of observations
for any time horizon $T$.

Each $\boldsymbol{x}_{t}$ consists of $D$ types of contexts, and can be written as
$\boldsymbol{x}_{t} = (x_{1,t}, x_{2,t}, \ldots, x_{D,t})$ where $x_{i,t}$ is
called the type $i$ context. ${\cal X}_{i}$ denotes the space of
type $i$ contexts and ${\cal X}:={\cal X}_{1}\times{\cal X}_{2}\times\ldots\times{\cal X}_{D}$ denotes
the space of context vectors. At any $t$, we have $x_{i,t}\in{\cal X}_{i}$
for all $i\in{\cal D}$. All of our results hold for the case when ${\cal X}_{i}$ is a bounded subset
of the real line. The number of elements in ${\cal X}_{i}$ can be finite or infinite. 
For the sake of notational simplicity we
take ${\cal X}_{i}=[0,1]$ for all $i\in{\cal D}$, since the values of context can be rescaled to lie in this range. 
Then, for the case when the actual context space is finite, $[0,1]$ will be a superset of the context space.  
For a context vector $\boldsymbol{x}$, $\boldsymbol{x}_{{\cal R}(a)}$ denotes the vector of values of $\boldsymbol{x}$ corresponding to types ${\cal R}(a)$. 
The reward of action $a$ for $\boldsymbol{x}= (x_1, x_2, \ldots, x_D) \in {\cal X}$, i.e.,
$r_{t}(a,\boldsymbol{x})$, is generated according to an i.i.d. process
with distribution $F(a,\boldsymbol{x}_{{\cal R}(a)})$ with support in $[0,1]$ and expected
value $\mu(a,\boldsymbol{x}_{{\cal R}(a)})$. 
The learner does not know $F(a,\boldsymbol{x}_{{\cal R}(a)})$ and $\mu(a,\boldsymbol{x}_{{\cal R}(a)})$ for $a\in{\cal A}$, $\boldsymbol{x}\in{\cal X}$ a priori.

The following assumption gives a similarity structure between the
expected reward of an action and the contexts of the type that is relevant to that action.
\begin{assumption*} (\textbf{The Similarity Assumption}) 
For all $a\in{\cal A}$,
$\boldsymbol{x},\boldsymbol{x}'\in{\cal X}$, we have 
$|\mu(a,\boldsymbol{x}_{{\cal R}(a)})-\mu(a,\boldsymbol{x}'_{{\cal R}(a)})|\leq L ||\boldsymbol{x}_{{\cal R}(a)}-\boldsymbol{x}'_{{\cal R}(a)} ||$,
where $L>0$ is the Lipschitz constant and $||\cdot||$ is the Euclidian norm. 
\end{assumption*}
We assume that the learner knows the $L$ given in the {\em Similarity Assumption}.
While we need this assumption in order to derive our analytic bounds on the performance of the algorithm, as it is common in all contextual bandit algorithms \cite{lu2010contextual,slivkins2011contextual}, our numerical results in Section \ref{sec:numerical} show that the proposed algorithm works well on real-world data sets for which this assumption may not hold. 
Given a context vector $\boldsymbol{x} = (x_1, x_2, \ldots, x_D)$, the optimal action is 
$a^{*}(\boldsymbol{x}):=\argmax_{a\in{\cal A}}\mu(a, \boldsymbol{x}_{{\cal R}(a)})$.
In order to assess the learner's loss due to unknowns, we compare its performance with the performance of an {\em oracle benchmark} which knows $a^{*}(\boldsymbol{x})$ for all $\boldsymbol{x}\in{\cal X}$.
Let $\mu_{t}(a):=\mu(a, \boldsymbol{x}_{{\cal R}(a),t})$. The action chosen by the learner
at time $t$ is denoted by $\alpha_{t}$.
The learner also decides whether to observe the reward or not, and this decision
of the learner at time $t$ is denoted by $\beta_{t}\in\{0,1\}$. If
$\beta_{t}=1$, then the learner chooses to observe the reward, else if
$\beta_{t}=0$, then the learner does not observe the reward.
The learner's performance loss with respect to the oracle benchmark
is defined as the regret, whose value at time $T$ is given by
\vspace{-0.1in}
\begin{align}
R(T) & :=\sum_{t=1}^{T}\mu_{t}(a^{*}(\boldsymbol{x}_{t}))-\sum_{t=1}^{T}(\mu_{t}(\alpha_{t})-c_{O}\beta_{t}).\label{eqn:regretdef}
\end{align}
%
Different from the definitions of regret in related works \cite{lu2010contextual,chu2011contextual,slivkins2011contextual}, there is an additional cost $c_{O}$, which is called the {\em active learning/exploration} cost. Hence the goal of the learner is to maximize its total reward while balancing the active learning costs incurred when observing the rewards. The algorithm we propose in this paper is able to achieve a given tradeoff between the two by actively controlling when to observe the rewards.

A regret that grows sublinearly in $T$, i.e., $O(T^{\gamma})$, $\gamma<1$, guarantees convergence in terms of the average reward, i.e., $R(T)/T \rightarrow 0$. We are interested in achieving sublinear growth with
a rate only depending on $D_{\textrm{rel}}$ independent of $D$.
\section{Online Learning of Relevance Relations}\label{sec:algodescription}
\subsection{Relevance Learning with Feedback} \label{sec:algodesc}

In this section we propose the algorithm {\em Relevance LEArning with Feedback} (RELEAF), which learns the best action for each context vector by simultaneously learning the relevance relation, and then estimating the expected
reward of each action based on the values of the contexts of the relevant types. The feedback, i.e., reward observations, is controlled based on the past context vector arrivals, in a way that the reward observations are only made for actions for which the uncertainty in the reward estimates are high for the current context vector. The controlled feedback feature allows RELEAF to operate as an active learning algorithm. 
RELEAF has a {\em relevance parameter} $\gamma_{\textrm{rel}}$ which is the number of relevant types it will learn for each action. 
In order to have analytic bounds on the regret, it is required that $\gamma_{\textrm{rel}} \geq D_{\textrm{rel}}$. However, the numerical results in Section \ref{sec:numerical} show that even with $\gamma_{\textrm{rel}} =1$, RELEAF performs very well on several real-world datasets.  
We assume that RELEAF knows $D_{\textrm{rel}}$ but not $\boldsymbol{{\cal R}}$. Hence, in this paper we assume that RELEAF is run with $\gamma_{\textrm{rel}} = D_{\textrm{rel}}$. In theory, it is enough for RELEAF to know an upper bound $\bar{D}_{\textrm{rel}}$ on $D_{\textrm{rel}}$. Then, the regret of RELEAF will depend on $\bar{D}_{\textrm{rel}}$. 
Operation of RELEAF can be summarized as follows:
\begin{itemize}
\item Adaptively form partitions (composed of intervals) of the context space of each type in ${\cal D}$ and use them to learn the action rewards of similar context vectors together from the history of observations.
\item For an action, form reward estimates for $2\gamma_{\textrm{rel}}$-tuple of intervals corresponding to $2\gamma_{\textrm{rel}}$-tuple of types. Based on the accuracy of these estimates, either choose to explore and observe the reward (by paying cost $c_O$ for active learning) or choose to exploit the best estimated action (but do not observe the reward) for the current context vector. 
\item In order to estimate the expected rewards of the actions accurately, find the set of $\gamma_{\textrm{rel}}$-tuple of types relevant to each action $a$. For instance, a $\gamma_{\textrm{rel}}$-tuple of types  $\boldsymbol{v} \in {\cal V}_{\gamma_{\textrm{rel}}}$ is relevant to action $a$ if ${\cal R}(a) \subset \boldsymbol{v}$. Conclude that $\boldsymbol{v}$ is relevant to $a$ if the variation of the reward estimates does not greatly exceed the natural variation of the expected reward of action $a$ over the hypercube corresponding to $\boldsymbol{v}$ formed by intervals of type $i \in \boldsymbol{v}$ (calculated using {\em Similarity Assumption}).
\end{itemize}

\begin{figure}[h!]
\fbox {
\begin{minipage}{0.95\columnwidth}
{\fontsize{9}{9}\selectfont
\flushleft{Relevance Learning with Feedback (RELEAF):}
\begin{algorithmic}[1]
\STATE{Input:  $L$, $\rho$, $\delta$, $\gamma_{\textrm{rel}}$.}
\STATE{Initialization: ${\cal P}_{i,1} = \{[0,1]\}$, $i \in {\cal D}$. Run {\bf Initialize}($i$, ${\cal P}_{i,1}$, $1$), $i \in {\cal D}$.}
\WHILE{$t \geq 1$}
\STATE{Observe $\boldsymbol{x}_t$, find $\boldsymbol{p}_t$ that $\boldsymbol{x}_t$ belongs to.}
\STATE{Set ${\cal U}_t := \bigcup_{i \in {\cal D}} {\cal U}_{i,t}$, where ${\cal U}_{i,t}$ (given in (\ref{eqn:underexplore})), is the set of under explored actions for type $i$.}
\IF{${\cal U}_t \neq \emptyset$}
\STATE{(\textbf{Explore}) $\beta_t =1$, select $\alpha_t$ randomly from ${\cal U}_t$, observe $r_t(\alpha_t, \boldsymbol{x}_t)$.}
\STATE{Update sample mean reward of $\alpha_t$ corresponding to $2\gamma_{\textrm{rel}}$-tuples of intervals: for all $\boldsymbol{q} \in Q_t$, given in (\ref{eqn:setofpairs}).\\
 $\bar{r}^{\boldsymbol{v}(\boldsymbol{q})}(\boldsymbol{q},\alpha_t) 
 = (S^{\boldsymbol{v}(\boldsymbol{q})}(\boldsymbol{q},\alpha_t) \bar{r}^{\boldsymbol{v}(\boldsymbol{q})}(\boldsymbol{q},\alpha_t) + r_t(\alpha_t, \boldsymbol{x}_t) ) /   ( S^{\boldsymbol{v}(\boldsymbol{q})} (\boldsymbol{q},\alpha_t) + 1 )$. 
}
\STATE{Update counters: for all $\boldsymbol{q} \in Q_t$, $S^{\boldsymbol{v}(\boldsymbol{q})}(\boldsymbol{q},\alpha_t) ++$. }
\ELSE
\STATE{(\textbf{Exploit}) $\beta_t =0$, for each $a \in {\cal A}$ calculate the set of candidate relevant contexts $\textrm{Rel}_t(a)$ given in (\ref{eqn:candrelevant}).}
\FOR{$a \in {\cal A}$}
\IF{$\textrm{Rel}_t(a) = \emptyset$}
\STATE{Randomly select $\hat{c}_t(a)$ from ${\cal V}_{\gamma_{\textrm{rel}}}$.}
\ELSE
\STATE{For each $i \in \textrm{Rel}_t(a)$, calculate $\textrm{Var}_t(\boldsymbol{v},a)$ given in  (\ref{eqn:canvar}).}
\STATE{Set $\hat{c}_t(a) = \argmin_{\boldsymbol{v} \in \textrm{Rel}_t(a) } \textrm{Var}_t(\boldsymbol{v},a)$.}
\ENDIF
\STATE{Calculate $\bar{r}^{\hat{c}_t(a)}(\boldsymbol{p}_{\hat{c}_t(a),t} a)$ as given in (\ref{eqn:equivalance}).}
\ENDFOR
\STATE{Select $\alpha_t = \argmax_{a \in {\cal A}} \bar{r}^{\hat{c}_t(a)}(\boldsymbol{p}_{\hat{c}_t(a),t} a)$.}
\ENDIF
\FOR{$i \in {\cal D}$}
\STATE{$N^i(p_{i,t})++$.}
\IF{$N^i(p_{i,t}) \geq 2^{\rho l(p_{i,t})} $}
\STATE{Create two new level $l(p_{i,t})+1$ intervals $p$, $p'$ whose union gives $p_{i,t}$.}
\STATE{${\cal P}_{i,t+1} = {\cal P}_{i,t} \cup \{ p, p' \} - \{ p_{i,t} \}$.}
\STATE{Run \textbf{Initialize}($i$, $\{ p, p' \}$, $t$).}
\ELSE
\STATE{${\cal P}_{i,t+1} = {\cal P}_{i,t}$.}
\ENDIF
\ENDFOR
\STATE{$t=t+1$}
\ENDWHILE
\end{algorithmic}
}
\end{minipage}
} 
\fbox {
\begin{minipage}{0.95\columnwidth}
{\fontsize{9}{9}\selectfont
{\bf Initialize}($i$, ${\cal B}$, $t$):
\begin{algorithmic}[1]
\FOR{$p \in {\cal B}$}
\STATE{\newc{Set $N^{i}(p)=0$, $\bar{r}^{ (\boldsymbol{v}(\boldsymbol{q}),i )}( (\boldsymbol{q},p) , a) = 0$, 
$S^{ (\boldsymbol{v}(\boldsymbol{q}),i ) } ( (\boldsymbol{q},p), a) = 0$ for all $2\gamma_{\textrm{rel}}$-tuple of types $(\boldsymbol{v}(\boldsymbol{q}),i )$ that contain type $i$, for all $a \in {\cal A}$  such that
$(\boldsymbol{q},p) \in \boldsymbol{{\cal P}}_{(\boldsymbol{v}(\boldsymbol{q}),i ),t}$.}}
\ENDFOR
\end{algorithmic}
}
\end{minipage}
}
\vspace{-0.1in}
\caption{Pseudocode for RELEAF.} \label{fig:CALIF}
\vspace{-0.2in}
\end{figure}

In order to learn fast, RELEAF exploits the similarities between
the context vectors of the relevant types\footnote{RELEAF only needs to know $L$ but not $\boldsymbol{{\cal R}}$. Even if $L$ is not known, it can use a slowly increasing function $\hat{L}(t)$ as an estimate for $L$ so that a sublinear regret bound will hold for a time horizon $T$ such that $\hat{L}(T) \geq L$.} given in the {\em Similarity Assumption} to estimate the rewards of the actions. The key to success of our algorithm is that this estimation is good enough if relevant tuples of types for each action are correctly identified. 
Since in Big Data applications $D$ can be very large, learning the $D_{\textrm{rel}}$-tuple of types that is relevant to each action greatly increases the learning speed.

RELEAF adaptively forms the partition
of the space for each type in ${\cal D}$, where the partition for
the context space of type $i$ at time $t$ is denoted by ${\cal P}_{i,t}$.
All the elements of ${\cal P}_{i,t}$ are disjoint intervals of ${\cal X}_{i}$
whose lengths are elements of the set $\{1,2^{-1},2^{-2},\ldots\}$.\footnote{Setting interval lengths to powers of $2$ is for presentational simplicity. In general, interval lengths can be set to powers of any integer greater than $1$.}
An interval with length $2^{-l}$, $l\geq0$ is called a level $l$
interval, and for an interval $p$, $l(p)$ denotes its level, $s(p)$ denotes its length. By
convention, intervals are of the form $(a,b]$, with the only exception
being the interval containing $0$, which is of the form $[0,b]$.\footnote{Endpoints of intervals will not matter in our analysis, so our results
will hold even when the intervals have common endpoints.} 
Let $p_{i,t}\in{\cal P}_{i,t}$ be the interval that $x_{i,t}$ belongs
to, $\boldsymbol{p}_{t}:=(p_{1,t},\ldots,p_{D,t})$ and $\boldsymbol{{\cal P}}_{t}:=({\cal P}_{1,t},\ldots,{\cal P}_{D,t})$. For $\boldsymbol{v} \in {\cal V}_K$, $1 \leq K \leq D$, let $\boldsymbol{p}_{\boldsymbol{v},t}$ denote the elements of $\boldsymbol{p}_t$ corresponding to types in $\boldsymbol{v}$, and let $\boldsymbol{{\cal P}}_{\boldsymbol{v}, t} = \times_{i \in \boldsymbol{v}} {\cal P}_{i,t}$.

The pseudocode of RELEAF is given in Fig. \ref{fig:CALIF}. 
RELEAF starts with ${\cal P}_{i,1} = \{  {\cal X}_i \} =\{ [0,1] \}$ for each $i \in {\cal D}$. 
As time goes on and more contexts arrive for each type $i$, it divides ${\cal X}_i$ into smaller and smaller intervals. Then, these intervals are used to create $2\gamma_{\textrm{rel}}$-dimensional hypercubes corresponding to $2\gamma_{\textrm{rel}}$-tuples of types, and past observations corresponding to context vectors lying in these hypercubes are used to form sample mean reward estimates of the expected action rewards.
The intervals are created in a way to balance the variation of the sample mean rewards due to the number of past observations that are used to calculate them and the variation of the expected rewards in each hypercube formed by the intervals. 
For each interval $p \in {\cal P}_{i,t}$, RELEAF keeps a counter for the number of type $i$ context arrivals to $p$. When the value of this counter exceeds $2^{\rho l(p)}$, where $\rho>0$ is an input of RELEAF called the {\em duration parameter}, $p$ is destroyed and two level $l(p)+1$ intervals, whose union gives $p$ are created. 
For example, when $p_{i,t} = (k 2^{-l}, (k+1) 2^{-l}]$ for some $0<k \leq 2^l -1$ if $N^i_t(p_{i,t}) \geq 2^{\rho l}$, RELEAF sets 
\begin{align}
&{\cal P}_{i,t+1} = {\cal P}_{i,t} - \{ p_{i,t}   \}  \notag \\
&\cup \{ (k 2^{-l}, (k+1/2) 2^{-l}  ], ( (k+1/2) 2^{-l} ,  (k+1) 2^{-l}]   \}   .  \notag
\end{align}
Otherwise ${\cal P}_{i,t+1}$ remains the same as ${\cal P}_{i,t}$.
It is easy to see that the lifetime of an interval increases exponentially in its duration parameter.

We next describe the control numbers RELEAF keeps for each type $i$, the counters and sample mean rewards RELEAF keeps for $2\gamma_{\textrm{rel}}$-tuples of intervals ($2\gamma_{\textrm{rel}}$-dimensional hypercubes) corresponding to a $2\gamma_{\textrm{rel}}$-tuple of types to determine whether to explore or exploit and how to exploit.
\newc{Let ${\cal V}_K(i)$ be the set of $K$-tuples of types that contains type $i$. For each $\boldsymbol{v} \in {\cal V}_K(i)$, we have $i \in \boldsymbol{v}$.}
 
\newc{Let ${\cal D}_{-\boldsymbol{v}} := {\cal D} - \{ \boldsymbol{v}\}$.}
For type $i$, let 
$Q_{i,t} := \{  \boldsymbol{p}_{\boldsymbol{v},t} : \boldsymbol{v} \in   {\cal V}_{2\gamma_{\textrm{rel}}}(i) \}$
be the set of $2\gamma_{\textrm{rel}}$-tuples of intervals that includes an interval belonging to type $i$ at time $t$, and let 
\begin{align}
Q_t := \bigcup_{i \in {\cal D}} Q_{i,t}  .    \label{eqn:setofpairs}
\end{align}
To denote an element of $Q_{i,t}$ or $Q_t$ we use index $\boldsymbol{q}$. 
For any $\boldsymbol{q} \in Q_t$, the tuple of types corresponding to the tuple intervals in $\boldsymbol{q}$ is denoted by $\boldsymbol{v}(\boldsymbol{q})$. For instance if 
$\boldsymbol{q} = (q_{i_1}, q_{i_2}, \ldots, q_{i_{2\gamma{\textrm{rel}}}})$, 
then $\boldsymbol{v}(\boldsymbol{q}) = (i_1, i_2, \ldots, i_{2\gamma{\textrm{rel}}})$.
The decision to explore or exploit at time $t$ is solely based on $\boldsymbol{p}_t$. 
For events $A_1, \ldots, A_K$, let $\mathrm{I}(A_1,\ldots, A_k)$ denote the indicator function of event $\bigcap_{k=1:K} A_k$. Let
\begin{align*}
S_{t}^{\boldsymbol{v}(\boldsymbol{q})}(\boldsymbol{q}, a) 
:=  \sum_{t'=1}^t \mathrm{I} \left( \alpha_{t'} =a , \beta_{t'} =1,  \boldsymbol{p}_{\boldsymbol{v}(\boldsymbol{q}),t'} = \boldsymbol{q} \right),
\end{align*}
be the number of times $a$ is selected and the reward is observed when the context values corresponding to types $\boldsymbol{v}(\boldsymbol{q})$ are in $\boldsymbol{q}$ and $\boldsymbol{q} \in \boldsymbol{{\cal P}}_{\boldsymbol{v}(\boldsymbol{q}),t}$.
Also let 
\begin{align}
& \bar{r}^{\boldsymbol{v}(\boldsymbol{q})}_t(\boldsymbol{q} ,a) \notag \\
& := \frac{\sum_{t'=1}^t r_{t'}(a, \boldsymbol{x}_{t'})  \mathrm{I} \left( \alpha_{t'} =a , \beta_t' =1, \boldsymbol{p}_{\boldsymbol{v}(\boldsymbol{q}),t'} = \boldsymbol{q} \right) }{S_{t}^{\boldsymbol{v}(\boldsymbol{q})}(\boldsymbol{q}, a) }, \notag
\end{align}
be the sample mean reward of action $a$ for $2\gamma_{\textrm{rel}}$-tuple of intervals $\boldsymbol{q}$.

At time $t$, RELEAF assigns a {\em control number} to each $i \in {\cal D}$ denoted by
\begin{align}
D_{i,t} := \frac{ 2 \log(t \newc{D^*} |{\cal A}| / \delta)}{ (L s(p_{i,t}) )^2} ,    \label{eqn:controlnum}
\end{align}
where 
\begin{align}
\newc{ D^* = {D-1 \choose 2\gamma_{\textrm{rel}} -1} } .   \label{eqn:newD}
\end{align}
This number depends on the cardinality of ${\cal A}$, the length of the active interval that type $i$ context is in at time $t$ and a {\em confidence parameter} $\delta>0$, which controls the accuracy of sample mean reward estimates. 
\newc{$D_{i,t}$ is a sufficient number of reward observations from an action, which guarantees that the estimated reward for that action will be sufficiently close to the expected reward for the context at time t. By sufficiently close we mean that when $i$ is the relevant type of context for the action, the difference between the true expected reward of that action and the estimated expected reward will be less than a constant factor of the length of the interval that contains the type $i$ context due to the Similarity Assumption. The control function ensures that within each hypercube, the rate of exploration only increases logarithmically in time. It also guarantees that each action is explored at least $\sim1/s(p_{i,t})^{2}$ times, which guarantees that the regret due to exploitations in each hypercube is small enough to achieve a sublinear regret bound (see Theorem 1).
}

Then, it computes the set of under-explored actions for type $i$ as
\begin{align}
 {\cal U}_{i,t} &:= \left\{  a \in {\cal A} : S^{\boldsymbol{v}(\boldsymbol{q})}_{t}( \boldsymbol{q},a) < D_{i,t} \right. \notag \\
&\left.  \textrm{ for some } \boldsymbol{q} \in Q_i(t)  \right\}   ,  \label{eqn:underexplore}
\end{align}
and then, the set of under-explored actions as ${\cal U}_t := \bigcup_{i \in {\cal D}} {\cal U}_{i,t}$. \newc{The decision to explore or exploit is based on whether or not ${\cal U}_t$ is empty, as follows:}

(i) If ${\cal U}_t  \neq \emptyset$, RELEAF randomly selects an action  $\alpha_t \in {\cal U}_t$ to explore, and observes its reward $r_t(\alpha_t, \boldsymbol{x}_t)$. Reward observation costs $c_O$, which is the active learning cost.
Then, it updates the sample mean rewards and counters for all $\boldsymbol{q} \in Q_t$,
\begin{align}
\bar{r}^{\boldsymbol{v}(\boldsymbol{q})}_{t+1}(\boldsymbol{q},\alpha_t) &=  \frac{ S^{\boldsymbol{v}(\boldsymbol{q})}_t(\boldsymbol{q},\alpha_t) \bar{r}^{\boldsymbol{v}(\boldsymbol{q})}_{t+1}(\boldsymbol{q},\alpha_t) +  r_t(\alpha_t, \boldsymbol{x}_t) }{S^{\boldsymbol{v}(\boldsymbol{q})}_t(\boldsymbol{q},\alpha_t) + 1}, \notag \\
S^{\boldsymbol{v}(\boldsymbol{q})}_{t+1}(\boldsymbol{q},\alpha_t)  &= S^{\boldsymbol{v}(\boldsymbol{q})}_t(\boldsymbol{q},\alpha_t) +1. \notag 
\end{align}

(ii) If ${\cal U}_t  = \emptyset$, RELEAF exploits by estimating the relevant $\gamma_{\textrm{rel}}$-tuple of types $\hat{c}_t(a)$ for each $a \in {\cal A}$ and forming sample mean reward estimates for action $a$ based on $\hat{c}_t(a)$. 
It first computes the set of {\em candidate relevant tuples of types} for each $a \in {\cal A}$. 
For each $\boldsymbol{v} \in {\cal V}_{\gamma_{\textrm{rel}}}$, let ${\cal V}_{2\gamma_{\textrm{rel}}}(\boldsymbol{v})$ 
be the set of $2\gamma_{\textrm{rel}}$-tuples of types such that $\boldsymbol{v} \cap \boldsymbol{w} = \boldsymbol{v}$ for $\boldsymbol{w} \in {\cal V}_{2\gamma_{\textrm{rel}}}(\boldsymbol{v})$. 

\begin{align}
\textrm{Rel}_t(a) &:=  \left\{  \boldsymbol{v} \in {\cal V}_{\gamma_{\textrm{rel}}} :
 | \bar{r}^{  \boldsymbol{w}    }_t( \boldsymbol{p}_{\boldsymbol{w}  ,t},a) - \bar{r}^{\boldsymbol{w}' }_t( \boldsymbol{p}_{\boldsymbol{w}' ,t}, a)    | \right. \notag \\
 & \left. \leq 3 L \sqrt{\gamma_{\textrm{rel}}}  \max_{i \in \boldsymbol{v}} s(p_{i,t}), \forall \boldsymbol{w}  , \boldsymbol{w}'  \in {\cal V}_{2\gamma_{\textrm{rel}}}(\boldsymbol{v}) \right\}  .  \label{eqn:candrelevant}
\end{align}
The intuition is that if the tuple of types $\boldsymbol{v}$ contains the tuple of types ${\cal R}(a)$ that is relevant to $a$, then independent of the values of the contexts of the other types, the variation of the pairwise sample mean reward of $a$ over $\boldsymbol{p}_{\boldsymbol{w},t}$ must be very close to the variation of the expected reward of $a$ in $\boldsymbol{p}_{\boldsymbol{v},t}$ for $\boldsymbol{w}  \in {\cal V}_{2D_{\textrm{rel}}}(\boldsymbol{v})$ in exploitation steps.

If $\textrm{Rel}_t(a)$ is empty, this implies that RELEAF failed to identify the relevant tuple of types, hence $\hat{c}_t(a)$ is randomly selected from ${\cal V}_{\gamma_{\textrm{rel}}}$. 
If $\textrm{Rel}_t(a)$ is nonempty, RELEAF computes the maximum variation
\begin{align}
\textrm{Var}_t(\boldsymbol{v},a) &:= \hspace{-0.1in} \max_{ \boldsymbol{w}  , \boldsymbol{w}'  \in {\cal V}_{2\gamma_{\textrm{rel}}}(\boldsymbol{v})  } | \bar{r}^{\boldsymbol{w}}_t(\boldsymbol{p}_{\boldsymbol{w},t}, a) - \bar{r}^{\boldsymbol{w}'}_t (\boldsymbol{p}_{\boldsymbol{w}',t} a)    | ,      \label{eqn:canvar}
\end{align}
for each $\boldsymbol{v} \in \textrm{Rel}_t(a)$. Then it sets $\hat{c}_t(a) = \min_{\boldsymbol{v} \in \textrm{Rel}_t(a)} \textrm{Var}_t(\boldsymbol{v},a)$. This way, whenever ${\cal R}(a) \subset \boldsymbol{v}$ for some $\boldsymbol{v} \in \textrm{Rel}_t(a)$, even if $\boldsymbol{v}$ is not selected as the estimated relevant tuple of types, the sample mean reward of $a$ calculated based on the estimated relevant tuple of types will be very close to the sample mean of its reward calculated according to ${\cal R}(a)$.
After finding the estimated relevant tuple of types $\hat{c}_t(a)$ for $a \in {\cal A}$, the sample mean rewards of the actions are computed as
\begin{align}
& \bar{r}^{\hat{c}_t(a)}_t ( \boldsymbol{p}_{\hat{c}_t(a),t}, a) \notag \\
& \hspace{-0.1in}  := \frac{  \hspace{-0.1in}  
\sum\limits_{\boldsymbol{w} \in {\cal V}_{2\gamma_{\textrm{rel}}}(\hat{c}_t(a))     }  \hspace{-0.15in} 
\bar{r}^{\boldsymbol{w}}_t ( \boldsymbol{p}_{\boldsymbol{w},t}, a)  S^{\boldsymbol{w}}_t (\boldsymbol{p}_{\boldsymbol{w},t}, a)  }
{  \sum\limits_{\boldsymbol{w} \in {\cal V}_{2\gamma_{\textrm{rel}}}(\hat{c}_t(a))   }  \hspace{-0.1in} 
S^{\boldsymbol{w} }_t ( \boldsymbol{p}_{\boldsymbol{w} ,t},a)} . \label{eqn:equivalance}   
\end{align}
Then, RELEAF selects 
\begin{align}
\alpha_t = \argmax_{a \in {\cal A}} \bar{r}^{\hat{c}_t(a)}_t ( \boldsymbol{p}_{\hat{c}_t(a),t}, a).  \notag
\end{align}
Different from explorations, since the reward is not observed in exploitations, sample mean rewards and counters are not updated.

\vspace{-0.1in}
\subsection{Why sample mean reward estimates for $2\gamma_{\textrm{rel}}$-tuple of intervals are required?}

Assume that RELEAF knows $D_{\textrm{rel}}$, hence $\gamma_{\textrm{rel}} = D_{\textrm{rel}}$.
Then, RELEAF computes sample mean reward estimates for $2D_{\textrm{rel}}$-tuples of intervals corresponding to different types and uses them to learn the action with the highest reward by learning the relevant $D_{\textrm{rel}}$-tuples of types. 
However, is it possible to learn the action with the highest reward by only forming sample
mean estimates for $D_{\textrm{rel}}$-tuples of intervals?
\newc{For instance consider the case when $D_{\textrm{rel}} =1$ and the following greedy learning algorithm called Greedy-RELEAF, outlined as follows:}

(i)  Form sample mean reward estimates of each action $a$ for each type $i \in {\cal D}$, i.e., $\bar{r}^i_t(p,a)$, $p \in {\cal P}_{i,t}$ based only on the context arrivals corresponding to type $i$; (ii) In exploitation steps choose the action with the highest sample mean reward over all sets of intervals in $\boldsymbol{p}_t$, i.e., $\argmax_{a \in {\cal A}} \max_{i \in \cal D} \bar{r}^i_t(p_i(t), a)$.
The following lemma shows that there exists a context arrival process for which the regret of Greedy-RELEAF will be linear in time. 

\begin{lemma} \label{lemma:counterexample}
Let ${\cal A} = \{ a,b \}$, ${\cal D} = \{ i, j\}$, ${\cal R}(a) =i$, ${\cal R}(b) =j$. $x_i(t) = x$ for all $t$ and $x_j(t) = 1$ with probability $0.8$ and $x_j(t) = 0$ with probability $0.2$ for all $t$ independently. Assume that $\mu(a,x) = 0.5$ and $\mu(b,x_j(t)) = x_j(t)$. Then, we have $R(T) = O(T)$. 
\end{lemma}
\begin{proof}
Given that Greedy-RELEAF explores sufficiently many times, at an exploitation step $t$ when the context vector is $(x,0)$, we have
\begin{align}
\mathrm{P} & \left(  |\bar{r}^i_t(p_i(t), a) -0.5   | < 0.1 , |\bar{r}^i_t(p_i(t), b) -0.8   | < 0.1, \right. \notag \\
& \left. |\bar{r}^j_t(p_j(t), a) - 0.5   | < 0.1  \right) \geq 0.5      \notag
\end{align}
for any $p_i(t)$ containing $x$  and $p_j(t)$ containing $0$. At such a $t$
Greedy-RELEAF will select action $b$ with probability at least $0.5$, resulting in an expected regret of at least $0.5^2$. Assume that the context vector arrivals are such that $(x,0)$ appears in more than $50\%$ of the time for all $T$ large enough. Then, the regret of Greedy-RELEAF will be linear in $T$.
\end{proof}

For the problem instance given in Lemma \ref{lemma:counterexample}, RELEAF will calculate and compare sample mean rewards $\bar{r}^{i,j}_t( (p_i(t),p_j(t)), a)$ for pairs of intervals corresponding to different types instead of directly forming sample mean rewards for intervals of each type; hence in exploitations it can identify that the type relevant to action $a$ is $i$ and action $b$ is $j$ with a very high probability. We will prove this in the following subsection by deriving a sublinear in time regret bound for RELEAF for the case when $D_{\textrm{rel}}=1$. \newc{A general regret bound for $1 \leq D_{\textrm{rel}} < D/2$ is proven in our online technical report \cite{tekinrelevance2014}.}

\vspace{-0.1in}
\subsection{Regret analysis of RELEAF for $D_{\textrm{rel}}=1$} \label{sec:regretRELEAF}

In this section we derive analytical regret bounds for RELEAF.
For simplicity of exposition, we prove our bounds for the special case when $D_{\textrm{rel}}=1$, i.e., when the relevance relation is a function, and RELEAF is run with $\gamma_{\textrm{rel}}=D_{\textrm{rel}}$. Although $D_{\textrm{rel}}=1$ is the simplest special case, our numerical results on real-world datasets in Section \ref{sec:numerical} shows that RELEAF performs very well with $\gamma_{\textrm{rel}}=1$. 

Let $\tau(T) \subset \{ 1, 2, \ldots, T \}$ be the set of time steps in which RELEAF exploits by time $T$. 
$\tau(T)$ is a random set which depends on context arrivals and the randomness of the action selection of RELEAF.
The regret $R(T)$ defined in (\ref{eqn:regretdef}) can be written as a sum of the regret incurred during explorations (denoted by $R_{\textrm{O}}(T)$) and the regret incurred during exploitations (denoted by $R_{\textrm{I}}(T)$).
Computing the two regrets separately gives more flexibility when choosing the parameter of RELEAF according to the objective of the learner. Although the definition of the regret in (\ref{eqn:regretdef}), allows us to write regret as $R_{\textrm{O}}(T) + R_{\textrm{I}}(T)$, the learner can set the parameters of RELEAF according to other objectives such as minimizing $R_{\textrm{I}}(T)$ subject to $R_{\textrm{O}}(T) \leq K$ for a fixed $T$ and $K>0$, or minimizing the time order of the regret when it is a more general function of regret in explorations and exploitations, i.e., $f(R_{\textrm{O}}(T) , R_{\textrm{I}}(T))$.
For instance, in an online prediction problem, if the cost of accessing the true label (exploration) is small, but the cost of making a prediction error in an exploitation step is very large, the learner can trade off to have higher rate of explorations. 

The following theorem gives a bound on the regret of RELEAF in exploitation steps. 

\begin{theorem} \label{thm:exploitbound}
Let RELEAF run with relevance parameter $\gamma_{\textrm{rel}} =1$, duration parameter $\rho>0$, confidence parameter $\delta>0$ and control numbers
\begin{align}
D_{i,t} := \frac{ 2 \log(t |{\cal A}| D / \delta)}{ (L s(p_{i,t}) )^2},  \notag
\end{align}
for $i \in {\cal D}$.
Let $R_{\textrm{inst}}(t)$ be the instantaneous regret at time $t$, which is the loss in expected reward at time $t$ due to not selecting $a^*(\boldsymbol{x}_t)$. When the relevance relation is such that $D_{\textrm{rel}} =1$,
then, with probability at least $1-\delta$, we have
\begin{align}
R_{\textrm{inst}}(t) \leq 8 L  ( s(p_{{\cal R}(\alpha_t),t}) + s(p_{{\cal R}(a^*(\boldsymbol{x}_t)),t}   )   ),  \notag
\end{align}
for all $t \in \tau(T)$, and the total regret in exploitation steps is bounded above by 
\begin{align}
R_{\textrm{I}}(T) &\leq 8 L \sum_{t \in \tau(T)} ( s(p_{{\cal R}(\alpha_t),t} + s(p_{{\cal R}(a^*(\boldsymbol{x}_t)),t}   )   )   \notag \\
&\leq 16 L D 2^{2\rho} T^{\rho/(1+\rho)} , \notag
\end{align} 
for arbitrary context vectors $\boldsymbol{x}_1, \boldsymbol{x}_2, \ldots, \boldsymbol{x}_T$.
Hence $R_{\textrm{I}}(T) /T = O(T^{-1/(1+\rho)})$, and $\lim_{T \to \infty} R_{\textrm{I}}(T)  =0$.
\end{theorem}
\begin{proof}
The proof is given in Appendix \ref{app:thm1}.
\end{proof}

Theorem \ref{thm:exploitbound} provides both context arrival process dependent and worst case bounds on the exploitation regret of RELEAF. By choosing $\rho$ arbitrarily close to zero, $R_{\textrm{I}}(T)$ can be made $O(T^\gamma)$ for any $\gamma >0$. While this is true, the reduction in regret for smaller $\rho$ not only comes from increased accuracy, but it is also due to the reduction in the number of time steps in which RELEAF exploits, i.e., $|\tau(T)|$. 
By definition, time $t$ is an exploitation step if 
\begin{align}
S^{(i,j)}_{t}(p_{i,t}, p_{j,t},a) &\geq \frac{ 2 \log(t |{\cal A}| D / \delta)}{ L^2 \min\{ s(p_{i,t})^2 , s(p_{j,t})^2 \} } \notag \\
& =\frac{ 2^{2 \max\{ l(p_{i,t}), l(p_{j,t}) \} +1} \log(t |{\cal A}| D / \delta)}{ L^2}  , \notag
\end{align}
for all $\boldsymbol{q} = (p_{i,t}, p_{j,t}) \in Q_{t}$, $i,j \in {\cal D}$. This implies that for any $\boldsymbol{q} \in Q_{i,t}$ which has the interval with maximum level equal to $l$, $\tilde{O}(2^{2l})$ explorations are required before any exploitation can take place. Since the time a level $l$ interval can stay active is $2^{\rho l}$, it is required that $\rho \geq 2$ so that $\tau(T)$ is nonempty.

The next theorem gives a bound on the regret of RELEAF in exploration steps.

\begin{theorem} \label{thm:exploreregret}
Let RELEAF run with $\gamma_{\textrm{rel}}$, $\rho$, $\delta$ and $D_{i,t}$, $i \in {\cal D}$ values as stated in Theorem \ref{thm:exploitbound}. When the relevance relation is such that $D_{\textrm{rel}} =1$, we have 
\begin{align}
R_{\textrm{O}}(T) &\leq \frac{ 960 D^2 (c_O+1) \log(T |{\cal A}| D/\delta)}{7 L^2} T^{4/\rho} \notag \\
&+ \frac{64 D^2 (c_O+1)}{3} T^{2/\rho} ,  \notag
\end{align}
with probability 1, for arbitrary context vectors $\boldsymbol{x}_1, \boldsymbol{x}_2, \ldots, \boldsymbol{x}_T$. 
Hence $R_{\textrm{O}}(T) /T = O(T^{(4-\rho)/\rho})$, and $\lim_{T \to \infty} R_{\textrm{O}}(T)   =0$ for $\rho >4$.
\end{theorem}
\begin{proof}
The proof is given in Appendix \ref{app:thm2}.
\end{proof}

Based on the choice of the duration parameter $\rho$, which determines how long an interval will stay active, it is possible to get different regret bounds for explorations and exploitations. Any $\rho>4$ will give a sublinear regret bound for both explorations and exploitations. The regret in exploitations increases in $\rho$ while the regret in explorations decreases in $\rho$. 

\begin{theorem} \label{theorem:regretbalance}
Let RELEAF run with $\gamma_{\textrm{rel}}$, $\delta$ and $D_{i,t}$, $i \in {\cal D}$ values as stated in Theorem \ref{thm:exploitbound} and $\rho = 2+ 2\sqrt{2}$. Then, the time order of exploration and exploitation regrets are balanced up to logarithmic orders. 
With probability at least $1-\delta$ we have both $R_{\textrm{I}}(T) = \tilde{O}(T^{2/(1+\sqrt{2})})$ and $R_{\textrm{O}}(T) = \tilde{O}(T^{2/(1+\sqrt{2})})$ .
\end{theorem}
\begin{proof}
The time order of the exploitation regret is increasing in $\rho$ from the result of Theorem \ref{thm:exploitbound}, and the time order of the exploration regret is decreasing in $\rho$ from the result of Theorem \ref{thm:exploreregret}. The time orders of both regrets are be balanced when 
$\rho/(1+\rho) = 4/\rho$, which gives the result.
\end{proof}

\comment{
\begin{remark}
Prior work on contextual bandits focused on balancing the regret due to exploration and exploitation. For example in \cite{lu2010contextual, slivkins2011contextual}, for a $D$-dimensional context vector algorithms are shown to achieve $\tilde{O}(T^{(D+1)/(D+2)})$ regret.\footnote{The results are shown in terms of the {\em covering dimension} which reduces to Euclidian dimension for our problem.} Also in \cite{ lu2010contextual} a $O(T^{(D+1)/(D+2)})$ lower bound on the regret is proved. 
An interesting question is to find the tightest lower bound for contextual bandits with relevance function.
One trivial lower bound is $O(T^{2/3})$, which corresponds to $D=1$. But the result in Theorem \ref{ass:similarity} says that the regret can be linear if the action rewards are estimated for each type separately from the context arrivals to the other types. 
Since comparisons between pairs of types are required to identify the type relevant to an action, the effective dimension of the problem is $2$. 

 However, since finding the relevant type for an action requires comparison of estimated rewards of the action with different relevant types, which requires accurate sample mean reward estimates for $2$ dimensions of the context space corresponding to those types, we conjecture that a tighter lower bound is $O(T^{3/4})$. Proving this is left as future work.
\end{remark}
}

Another interesting case is when actions with suboptimality greater than $\epsilon>0$ must never be chosen in any exploitation step by time $T$. When such a condition is imposed, RELEAF can start with partitions ${\cal P}_{i,1}$ that have intervals with high levels such that it explores more at the beginning to have more accurate reward estimates before any exploitation. The following theorem gives the regret bound of RELEAF for this case.

\begin{theorem} \label{thm:epsilonoptimal}
Let RELEAF run with relevance parameter $\gamma_{\textrm{rel}}=1$, duration parameter $\rho>0$, confidence parameter $\delta>0$, control numbers
\begin{align}
D_{i,t} := \frac{ 2 \log(t |{\cal A}| D / \delta)}{ (L s(p_{i,t}) )^2},   \notag
\end{align}
and with initial partitions ${\cal P}_{i,1}$, $i \in {\cal D}$ consisting of intervals with levels $l_{\min} = \lceil \log_2(3L/(2\epsilon)) \rceil$. 
When the relevance relation is such that $D_{\textrm{rel}} =1$, then, with probability $1-\delta$, we have
\begin{align}
R_{\textrm{inst}}(t) \leq \epsilon, \notag
\end{align}
for all $t \in \tau(T)$, 
\begin{align}
R_{\textrm{I}}(T) \leq 16 L 2^{2\rho} T^{\rho/(1+\rho)}, \notag
\end{align}
 and 
\begin{align}
R_{\textrm{O}}(T) &\leq \frac{81L^4}{\epsilon^4} 
\left( \frac{ 960 D^2 (c_O+1) \log(T|{\cal A}| D/\delta)}{7 L^2} T^{4/\rho}  \right. \notag \\
&\left. + \frac{64 D^2 (c_O+1)}{3} T^{2/\rho} \right) ,  \notag
\end{align}
for arbitrary context vectors $\boldsymbol{x}_1, \boldsymbol{x}_2, \ldots, \boldsymbol{x}_T$. Bounds on $R_{\textrm{I}}(T)$ and $R_{\textrm{O}}(T)$ are balanced for $\rho = 2+ 2\sqrt{2}$.
\end{theorem}
\begin{proof}
The proof is given in Appendix \ref{app:thm4}.
\end{proof}

\subsection{Regret bound for RELEAF for $D_{\textrm{rel}} < D/2$} \label{sec:regretRELEAF2}
\newc{
Similar to the analysis in the previous subsection, RELEAF achieves sublinear in $D_{\textrm{rel}}$ regret for any $D_{\textrm{rel}}<  D/2$.

\begin{theorem} \label{thm:regretgeneral}
Let RELEAF run with relevance parameter $\gamma_{\textrm{rel}} =D_{\textrm{rel}}$, duration parameter $\rho>0$, confidence parameter $\delta>0$ and control numbers
\begin{align}
D_{i,t} := \frac{ 2 \log(t |{\cal A}| D^* / \delta)}{ (L s(p_{i,t}) )^2},  \notag
\end{align}
for $i \in {\cal D}$, where $D^*$ is given in (\ref{eqn:newD}).
Then, 
with probability at least $1-\delta$ we have
$R_{\textrm{I}}(T)  =  \tilde{O} ( T^{ g(D_{\textrm{rel}}) } )$
and
$R_{\textrm{O}}(T)  = \tilde{O} (T^{ g(D_{\textrm{rel}}) } )$, 
where 
\begin{align}
g(D_{\textrm{rel}})  :=    \frac{2+2D_{\textrm{rel}} + \sqrt{ 4 D^2_{\textrm{rel}} + 16 D_{\textrm{rel}} +12 } }  
  { 4+ 2 D_{\textrm{rel}} + \sqrt{ 4 D^2_{\textrm{rel}} + 16 D_{\textrm{rel}} +12 }} .   \notag
\end{align}
\end{theorem}
\begin{proof}
\aremove{The proof is given in our online technical report \cite{tekinrelevance2014}.}
\jremove{The proof is given in Appendix \ref{app:thm5}.}
\end{proof}

The bound on the regret given in Theorem \ref{thm:regretgeneral} matches the bound in Theorem \ref{theorem:regretbalance} for $D_{\textrm{rel}}=1$. 

\begin{remark}\label{remark:compare}
The regret bound in Theorem \ref{thm:regretgeneral} is better than the {\em generic} regret bound $\tilde{O}(T^{(D+1)/(D+2)})$ for contextual bandit algorithms \cite{slivkins2011contextual,lu2010contextual} that does not exploit the existence of relevance relations when $D_{\textrm{rel}} \leq D/2 - 1$.
\end{remark}
}

\vspace{-0.1in}
\section{Numerical Results}
\label{sec:numerical}

\jremove{In this section, we numerically compare the performance of our learning algorithm with state--of--the--art learning techniques, including ensemble learning methods and other multi-armed bandit algorithms for three different real-world datasets: (i) breast cancer diagnosis, (ii) network intrusion detection, (iii) webpage recommendation.
The purpose of simulations for the first two datasets is to show that RELEAF can learn to make accurate prediction without the need of base classifiers, which are required by ensemble learners. 
The purpose of simulations for the third dataset is to show that RELEAF can learn to make accurate recommendations based on the context vectors of the users, by only observing the click information for the recommended webpage.}

\aremove{In this section, we numerically compare the performance of our learning algorithm with state--of--the--art learning techniques, including ensemble learning methods and other multi-armed bandit algorithms for two real-world datasets: (i) network intrusion detection, (ii) webpage recommendation.
The purpose of simulations for the first dataset is to show that RELEAF can learn to make accurate prediction without the need of base classifiers, which are required by ensemble learners. 
The purpose of simulations for the second dataset is to show that RELEAF can learn to make accurate recommendations based on the context vectors of the users, by only observing the click information for the recommended webpage.
An extended numerical results section, which includes additional information about the datasets and additional simulation results can be found in our online technical report \cite{tekinrelevance2014}.}



\vspace{-0.1in}
\subsection{Datasets}
\label{sec:sets}

\jremove{
\noindent \textbf{Breast Cancer (BC)} \cite{UCI}:
The dataset consists of features extracted from the images of fine needle aspirate (FNA) of breast mass, that gives information about the size, shape, uniformity, etc., of the cells. Each feature has a finite number of values that it can take, and the values of features are normalized\footnote{Normalization is done in the following way: maximum and minimum context values in the dataset are found. Minimum context value is subtracted from all contexts, then the result is divided by the difference between the maximum and minimum values} such that they lie in $[0,1]$.
Each case is labeled either as ``malignant" or ``benign".
We assume that images arrive to the learner in an online fashion.
At each time slot, the learning algorithm operates on a $9$ dimensional feature vector which consists of a subset of the features extracted from the same image.

The prediction action belongs to the set $\{ benign, malignant \}$. Reward is $1$ when the prediction is correct and $0$ else. 
50000 instances are created by duplication of the data and are randomly sequenced. 
Out of these 69\% of the instances are labeled as ``benign" while the rest is ``labeled" as malignant.
}

\noindent \textbf{Network Intrusion (NI)} \cite{UCI}: The network intrusion dataset from UCI archive \cite{UCI} consists of a series of TCP connection records, labeled either as normal connections or as attacks.
The data consists of 42 features, and we take 15 of them as types of contexts. Taken features are normalized to lie in $[0,1]$.
The prediction action belongs to the set $\{ attack, noattack \}$. Reward is $1$ when the prediction is correct and $0$ otherwise. 

\noindent \textbf{Webpage Recommendation (WR)} \cite{li2010contextual}: This dataset contains webpage recommendations of Yahoo! Front Page which is an Internet news website. 
Each instance of this dataset consists of (i) IDs of the recommended items and their features, (ii) context vector of the user, and (iii) user click information. For a recommended webpage (item), reward is $1$ if the user clicks on the item and $0$ otherwise.
The context vector for each user is generated by mapping a higher dimensional set of features of the user including features such as gender, age, purchase history, etc. to $[0,1]^5$. The details of this mapping is given in \cite{li2010contextual}.
We select 5 items and consider $T=10000$ user arrivals.

\subsection{Learning algorithms} \label{sec:comparemethods}

Next we briefly summarize the algorithms considered in our evaluation:

\comment{
\textbf{Random}: Selects a random action at each time step. 

\textbf{UCB1} \cite{auer2002finite}: This is a multi-armed bandit algorithm that assigns an index to each arm based on its confidence level and selects the arm with the highest index. It does not exploit the availability of context information. 

\textbf{$\epsilon$-greedy} \cite{auer2002finite}:
This algorithm keeps sample mean reward estimates for each action. At time step $t$, the action with the highest sample mean reward is selected with probability $\epsilon_t$, or another action is randomly explored with probability $1-\epsilon_t$, where $\epsilon_t \sim 1/t$. Similar to UCB1, context information is not taken into account.

\textbf{Greedy Contextual Learning (GCL)}: This algorithm forms an adaptive partition of the context space the same way as RELEAF. For an action $a \in {\cal A}$, separate sample mean reward estimates are formed for each type $d \in {\cal D}$, for each active interval of that type. At an exploitation step, the action with the highest sample mean reward among all types is selected.
}

\textbf{RELEAF}:  Our algorithm given in Fig. \ref{fig:CALIF} with control numbers $D_{i,t}$ divided by $5000$ to reduce the number of explorations.\footnote{The theoretical bounds are proven to hold for worst-case context vector arrivals and reward distributions. In practice, the relevance relation and the order of action rewards are identified correctly with much less explorations.}

\textbf{RELEAF-ALL}:  Same as RELEAF except that reward of the selected action is observed in every time step. This version is useful when the reward of the selected action can be observed with no cost.

\textbf{RELEAF-FO}: Same as RELEAF except that it observes the rewards of all actions instead of the reward of the selected action. We refer to this version of our algorithm as RELEAF with full observation
(RELEAF-FO). 

\textbf{Contextual zooming (CZ)} \cite{slivkins2011contextual}: This algorithm adaptively creates balls over the joint action and context space, calculates an index for each ball based on the history of selections of that ball, and at each time step selects an action according to the ball with the highest index that contains the action-context pair. 

\textbf{Hybrid-$\epsilon$} \cite{bouneffouf2012hybrid}:
This algorithm is the contextual version of $\epsilon$-greedy, which forms context-dependent sample mean rewards for the actions by considering the history of observations and decisions for groups of contexts that are similar to each other.

\textbf{LinUCB} \cite{li2010contextual}: This algorithm computes an index for each action by assuming that the expected reward of an action is a linear combination of different types of contexts. The action with the highest index is selected at each time step. 

\textbf{Ensemble Learning Methods} Average Majority (AM) \cite{gao2007appropriate}, Adaboost \cite{freund1995desicion}, Online Adaboost \cite{fan1999application} and Blum's Variant of Weighted Majority (Blum) \cite{blum1997empirical}: The goal of ensemble learning is to create a strong (high accuracy) classifier by combining predictions of base classifiers. Hence all these methods require base classifiers (trained a priori) that produce predictions (or actions) based on the context vector. 

AM simply follows the prediction of the majority of the classifiers and does not perform active learning.
Adaboost is trained a priori with 1500 instances, whose labels are used to compute the weight vector.
Its weight vector is fixed during the test phase (it is not learning online); hence no active learning is performed during the test phase.
In contrast, Online Adaboost always receives the true label at the end of each time slot. 
It uses a time window of $1000$ past observations to retrain its weight vector.
Similar to Online Adaboost, Blum also learns its weight vector online. 
The key differences between our algorithm and the methods that we compare against are given in Table \ref{tab:numericaldifferences}.

\begin{table*}
\centering
{\fontsize{9}{9}\selectfont
\setlength{\tabcolsep}{.1em}
\begin{tabular}{|c|c|c|c|c|c|}
\hline
\textbf{Algorithm}  & \textbf{Base classifiers} & \textbf{Prior training} & \textbf{Online Learning} & \textbf{Active learning} \\
\hline
\textbf{AM} \cite{gao2007appropriate}   & required  & no & no & no  \\
\hline
\textbf{Adaboost} \cite{freund1995desicion} &  required & required & no &  no \\ 
\hline
\textbf{Online Adaboost} \cite{freund1995desicion}, \textbf{Blum} \cite{blum1997empirical}  & required & required & yes & no  \\
\hline
\textbf{CZ} \cite{slivkins2011contextual}, \textbf{Hybrid-$\epsilon$} \cite{bouneffouf2012hybrid} , \textbf{LinUCB} \cite{li2010contextual}  & not required & not required & yes & no \\
\hline
\textbf{RELEAF}  &  not required & not required & yes & yes \\
\hline
\end{tabular}
}
\caption{Properties of RELEAF, ensemble learning methods and other contextual bandit algorithms.}
\label{tab:numericaldifferences}
\vspace{-0.2in}
\end{table*}

\jremove{\subsection{Breast cancer simulations}

In this section we compare the performance of RELEAF, RELEAF-ALL and RELEAF-FO with other learning methods described in Section \ref{sec:comparemethods}.
For the ensemble learning methods, there are 6 logistic regression base classifiers, each trained with a different set of 10 instances.

The simulation results are given in Table \ref{tab:schemes}.
Since RELEAF-FO updates the reward of both predictions after the label is received, it achieves lower error rates compared to RELEAF. In this setting it is natural to assume that the reward of both predictions are updated, because observing the label gives information about which prediction is correct. RELEAF-ALL which observes all the labels has the lowest error rate. 

Among the ensemble learning schemes Adaboost and Online Adaboost performs the best, however, their error rates are more than two times higher than the error rate of RELEAF and about three times higher than the error rate of RELEAF-FO. Although the number of actively obtained labels (explorations) for RELEAF and RELEAF-FO are higher than the initial training samples used to train Adaboost; neither RELEAF nor RELEAF-FO has a predetermined exploration size as Adaboost. This is especially beneficial when time horizon of interest is unknown or prediction performance is desired to be uniformly good over all time instances. 
CZ is the best among the other multi-armed bandit algorithms with $3.15\%$ error, but worse than RELEAF which has $1.88\%$ error.

\begin{table*}
\centering
{\fontsize{9}{9}\selectfont
\setlength{\tabcolsep}{.1em}
\begin{tabular}{|c|c|c|c|c|c|c|c|c|c|c|}
\hline
\multirow{3}{*}{\textbf{Algorithm}} &   \multicolumn{5}{|c|}{\textbf{Performance}} \\ \cline{2-6}  
&   error \% & missed \% & false \%  & number of & active learning \\ 
&  & & & label observations & cost for $c_O =1$ \\
\hline
\textbf{AM} & 8.22 & 17.20 & 4.09 & 0 (no online learning) & 0 \\ 
\hline
\textbf{Adaboost} &  4.60  & 3,82  & 4.97 & 1500 (to train weights) & 1500  \\
\hline
\textbf{Online Adaboost} &  4.68  & 4.07 & 4.95 & all labels are observed & 50000  \\
\hline
\textbf{Blum} &  11.18 & 27.12 & 3.86  & all labels are observed & 50000 \\ 
\hline 
\textbf{CZ}  & 3.15 & 4.24 & 2.89 & all labels are observed & 50000  \\ 
\hline
\textbf{Hybrid-$\epsilon$}  &   8.83 & 11.77 & 7.48 & all labels are observed & 50000  \\ 
\hline
\textbf{LinUCB}  & 10.67 & 7.27 & 12.22 & all labels are observed & 50000 \\ 
 &  & & & & \\
\hline
\textbf{RELEAF} &   1.88 & 1.93  & 1.86 & 2630 & 2630 \\ 
\hline
\textbf{RELEAF-ALL}  &  1.24 & 1.19 & 1.36 & all labels are observed & 50000 \\ 
\hline
\textbf{RELEAF-FO}    & 1.68 & 1.34 & 1.82 & 2630 & 2630  \\ 
\hline
\end{tabular}
}
\caption{Comparison of RELEAF with ensemble learning methods and other contextual bandit algorithms for the breast cancer dataset.}
\vspace{-0.3in}
\label{tab:schemes}
\end{table*}
}

\subsection{Network intrusion simulations}

In this section we compare the performance of RELEAF, RELEAF-ALL and RELEAF-FO with other learning methods described in Section \ref{sec:comparemethods}. For the ensemble learning methods, the base classifiers are logistic regression classifiers, each trained with $5000$ different instances from the NI. Comparison of performances in terms of the error rate is given in Table \ref{tab:schemeintrusion}. We see that RELEAF-FO has the lowest error rate at $0.68\%$, more than two times better than any of the ensemble learning methods. 
All the ensemble learning methods we compare against use classifiers to make predictions, and these classifiers require a priori training. In contrast, RELEAF and RELEAF-FO
do not require any a priori training, learn online and
require only a small number of label observations (i.e. they can perform active
learning).

CZ performs very poorly in this simulation because its learning rate is sensitive to Lipschitz constant that is given as an input to the algorithm which we set equal to $0.5$. Numerical results related to the performance of CZ and RELEAF for different $L$ values can be found in our online technical report \cite{tekinrelevance2014}. 
LinUCB performs the best in terms of the overall rate of error, but if we consider the error rate of RELEAF in exploitations it is better than LinUCB. This highlights the finding of Theorem \ref{thm:exploitbound} regarding RELEAF, which states that highly suboptimal actions are not chosen in exploitations with a high probability.

\begin{table}[h]
\centering
{\fontsize{9}{7}\selectfont
\setlength{\tabcolsep}{.1em}
\begin{tabular}{|c|c|c|c|}
\hline
\textbf{Algorithm} &    error \% & exploitation  & number of \\
& & error  \% & label observations \\
\hline
\textbf{AM} &  3.07 & N/A & 0 \\ 
\hline
\textbf{Adaboost} &   3.1  & N/A & 1500 \\
\hline
\textbf{Online} &   2.25 & N/A & all \\
\textbf{Adaboost} &  & & \\
\hline
\textbf{Blum}  &  1.64 & N/A & all \\ 
\hline 
\textbf{CZ}  &  53 & N/A & all  \\ 
\hline
\textbf{Hybrid-$\epsilon$}  &   8.8  & N/A & all   \\ 
\hline
\textbf{LinUCB}  &  0.27 & N/A & all   \\ 
\hline
\textbf{RELEAF}  &   1.19 & 0.24 & 398 \\ 
\hline
\textbf{RELEAF-ALL}  &   1.07 & 0.22 & all  \\ 
\hline
\textbf{RELEAF-FO}  &  0.68 & 0.24 & 229  \\ 
\hline
\end{tabular}
}
\caption{Comparison of the error rates of RELEAF-FO with ensemble learning methods for network intrusion dataset.}
\label{tab:schemeintrusion}
\vspace{-0.2in}
\end{table}

\subsection{Webpage recommendation simulations}

In this dataset only the click behavior of the user for the recommended item is observed. Moreover, it is reasonable to assume that the click behavior feedback is always available (no costly observations). 
The ensemble learning methods require availability of experts recommending actions and full reward feedback including the rewards of the actions that are not selected, to update the weights of the experts, hence they are not suitable for this dataset. In contrast, multi-armed bandit methods are more suitable since only the feedback about the reward of the chosen action is required. Hence we only compare RELEAF-ALL, CZ, LinUCB and Hybrid-$\epsilon$ for this dataset. 
We compare the click through rates (CTRs), i.e., average number of times the recommended item is clicked, of all algorithms in Table \ref{tab:schemeyahoo}.
We observe that RELEAF-ALL has the highest CTR.

\jremove{
\subsection{Identifying the relevant types}
\newc{
When RELEAF exploits at time $t$, it identifies a relevant type $\hat{c}_t(a)$ for every action $a \in {\cal A}$ and selects the arm with the highest sample mean reward according to its estimated relevant type. Hence, the value of the context of the relevant type plays an important role on how well RELEAF performs. 

For each dataset we choose a single action and for each chosen action show in Table \ref{tab:relevancefreq} the percentage of times a type is selected as the type that is relevant to that action in the time slots that RELEAF exploits. Since there are many types, only the 4 of the types which are selected as the relevant type for the corresponding action highest number of times are shown. For instance, for BC in $70\%$ of the exploitation slots the type identified as the type relevant to action ``predict benign" comes from a $3$ element subset of the set of $9$ types in the data. Similarly for NI the type identified as the type relevant to action ``predict attack" comes from a $2$ element subset of the set of $15$ types in the data for $85\%$ of the exploitation slots. 

This information provided by RELEAF can be used to identify the relevance relation that is present in a dataset. For instance, consider the NI dataset. Since the type that is assigned as the estimated relevant type most of the times is only assigned in $45\%$ of the exploitation slots, for the NI dataset we should have $D_{\textrm{rel}}>1$. However, since the pair of types that are assigned as the estimated relevant type most of the times is assigned in $85\%$ of the exploitation slots, we can conclude that approximately $D_{\textrm{rel}}\leq2$ for the NI dataset.
}
}



\begin{table}
\centering
{\fontsize{9}{7}\selectfont
\setlength{\tabcolsep}{.1em}
\begin{tabular}{|c|c|}
\hline
\textbf{Abbreviation} &   CTR  \\
\hline
\textbf{CZ} & 3.79 \\ 
\hline
\textbf{Hybrid-$\epsilon$} &   6.41  \\
\hline
\textbf{LinUCB} &   6.06 \\
\hline
\textbf{RELEAF-ALL}  &  6.62  \\ 
\hline
\end{tabular}
}
\caption{Comparison of the click through rates (CTRs) of RELEAF, CZ, Hybrid-$\epsilon$ and LinUCB for webpage recommendation dataset.}
\label{tab:schemeyahoo}
\vspace{-0.2in}
\end{table}

\jremove{
\begin{table}
\centering
{\fontsize{9}{9}\selectfont
\setlength{\tabcolsep}{.1em}
\begin{tabular}{|c|c|c|c|c|c|}
\hline
\multirow{2}{*}{Dataset} &  \multirow{2}{*}{Action} &  \multicolumn{4}{|c|} {highest rates of relevance}  \\
\cline{3-6}
   & & highest  & 2nd highest  & 3rd highest  & 4th highest  \\
   &  & type-rate & type-rate & type-rate &  type-rate \\
\hline
\textbf{BC} & predict ``benign" & 3-27\%  & 1-22\% & 7-21\% & 2-12\% \\
\hline
\textbf{NI} & predict ``attack" & 1-45\% & 15-40\%  & 2-7\% & 4-5\% \\
\hline
\textbf{WR} & recommend  &  3-46\% & 1-44\% & 2-8\% & 4-1\% \\ 
& webpage $a$ & & & & \\
\hline
\textbf{WR} & recommend  &  2-57\% & 1-32\% & 5-9\% & 4-1\% \\ 
& webpage $b$ & & & & \\
\hline
\end{tabular}
}
\caption{Average number of times RELEAF identified a type as the type relevant to the specified action in exploitations.}
\label{tab:relevancefreq}
\end{table}
}

\section{Conclusion}\label{sec:conc}

In this paper we formalized the problem of learning the best
action (prediction, recommendation etc.) to be taken based on the current streaming Big Data by online learning the relevance relation between types of
contexts and actions.
We proposed an algorithm that (i) has sublinear regret with time order independent
of $D$, (ii) only requires reward observations in explorations, (iii)
for any $\epsilon>0$, does not select any $\epsilon$ suboptimal actions
in exploitations with a high probability. 
We illustrated the properties of the proposed algorithm via extensive numerical simulations on real-data, showed that it achieves high average reward and identifies the set of relevant types. 
The proposed algorithm can be used in a variety of application (including applications requiring active learning) such as medical diagnosis, recommender systems and stream mining problems.
An interesting future research direction is learning both relevant types of contexts and relevant type of actions for multi-armed bandit problems with high dimensional action and context spaces.


\bibliographystyle{IEEE}
\bibliography{NIPSrelevant}

\appendices

\section{Proof of Theorem 1} \label{app:thm1}

Let $A := |{\cal A}|$. We first define a sequence of events which will be used in the analysis of the regret of RELEAF.
For $p \in {\cal P}_{{\cal R}(a),t}$, Let $\pi(a, p)  = \mu(a, x^*_{{\cal R}(a)}(p) )$, where $x^*_{{\cal R}(a)}(p)$ is the context at the geometric center of $p$. For $j \in {\cal D}_{-{\cal R}(a)}$, let
\begin{align}
\textrm{INACC}_t(a,j)  
& := \left\{  |\bar{r}^{({\cal R}(a),j)}_t ( (p_{{\cal R}(a),t}, p_{j,t}), a) - \pi(a, p_{{\cal R}(a),t})  | \right. \notag \\
& \left.  > \frac{3}{2} L s(p_{{\cal R}(a),t})    \right\} ,   \notag
\end{align}
be the event that the pairwise sample mean corresponding to pair $({\cal R}(a),j)$ of types is {\em inaccurate} for action $a$. 
Let 
$\textrm{ACC}_t(a) := \bigcap_{j \in {\cal D}_{-{\cal R}(a)}}  \textrm{INACC}_t(a,j)^C$,
be the event that all pairwise sample means corresponding to pairs $({\cal R}(a),j)$, $j \in {\cal D}_{-{\cal R}(a)}$ are accurate. 
Consider $t \in \tau(T)$. 
Let
$\textrm{WNG}_t(a) :=  \left\{  {\cal R}(a) \notin \textrm{Rel}_t(a)  \right\}$,
be the event that the type relevant to action $a$ is not in the set of candidate relevant types, and
$\textrm{WNG}_t := \bigcup_{a \in {\cal A}} \textrm{WNG}_t(a)$,
be the event that the type relevant to some action $a$ is not in the set of candidate relevant types of that action.
Finally, let 
$\textrm{CORR}_T := \bigcap_{t \in \tau(T)}   \textrm{WNG}_t^C$,     \notag
be the event that the relevant types for all actions are in the set of candidate relevant types at all exploitation steps.

We first prove several lemmas related to Theorem 1. The next lemma gives a lower bound on the probability of $\textrm{CORR}_T$.

\begin{lemma} \label{lemma:corr}
For RELEAF,
for all $a \in {\cal A}$, $t \in \tau(T)$, 
we have 
$\mathrm{P}(\textrm{INACC}_t(a,j) )  
\leq \frac{2\delta}{AD t^{4}}$. 
for all $j \in {\cal D}_{-{\cal R}(a)}$, and
$\mathrm{P} (\textrm{CORR}_T) \geq 1-\delta$ for any $T$.
\end{lemma}
\begin{proof}
For $t \in \tau(T)$, we have ${\cal U}_t = \emptyset$, hence
\begin{align}
S^{\boldsymbol{v}(\boldsymbol{q})}_t(\boldsymbol{q},a) \geq \frac{2 \log(t A D / \delta)}{ (L s(p_{{\cal R}(a),t}))^2 } ,      \notag
\end{align}
for all $a \in {\cal A}$, $\boldsymbol{q} \in Q_i(t)$ and $i \in {\cal D}$. 
Due to the {\em Similarity Assumption}, since rewards in $\bar{r}^{({\cal R}(a),j)}_t ( (p_{{\cal R}(a),t}, p_{j,t}), a)$ are sampled from distributions with mean between $[\pi(a, p_{{\cal R}(a),t}) - \frac{L}{2} s(p_{{\cal R}(a),t}), \pi(a, p_{{\cal R}(a),t}) + \frac{L}{2} s(p_{{\cal R}(a),t}) ]$, using a Chernoff bound we get
\begin{align}
\mathrm{P}(\textrm{INACC}_t(a,j) ) 
& \leq 2 \exp\left(  -2 (L s(p_{{\cal R}(a),t}))^2 \frac{2 \log(tAD/\delta)}{(L s(p_{{\cal R}(a),t}))^2}        \right)   \notag \\
& \leq 2\delta / (AD t^{4}) .  \notag
\end{align}
We have $\textrm{WNG}_t(a) \subset \bigcup_{j \in {\cal D}_{-{\cal R}(a)}} \textrm{INACC}_t(a,j)$.
Thus
\begin{align}
\mathrm{P}(   \textrm{WNG}_t(a) ) \leq  2\delta/(A t^{4}) , \textrm { and }
\mathrm{P}(   \textrm{WNG}_t ) \leq  2\delta/ t^{4} . \notag
\end{align}
This implies that 
\begin{align}
\mathrm{P} (\textrm{CORR}_T^C) &\leq \sum_{t \in \tau(T)}  \mathrm{P}( \textrm{WNG}_t )   \\
&\leq \sum_{t \in \tau(T)} \frac{2\delta}{ t^{4}} \leq \sum_{t=3}^\infty \frac{2\delta}{ t^{4}}  \leq \delta. \notag
\end{align}
\end{proof}

\begin{lemma} \label{lemma:estimatedrel}
When $\textrm{CORR}_T$ happens we have for all $t \in \tau(T)$
\begin{align}
|\bar{r}^{\hat{c}_t(a)}_t (p_{\hat{c}_t(a),t}, a) - \mu(a, x_{{\cal R}(a),t}) |  \leq 8 L s(p_{{\cal R}(a),t}) .  \notag
\end{align}
\end{lemma}
\begin{proof}
From Lemma \ref{lemma:corr}, $\textrm{CORR}_T$ happens when
\begin{align}
|\bar{r}^{({\cal R}(a),j)}_t((p_{{\cal R}(a),t}, p_{j,t}), a) - \pi(a, p_{{\cal R}(a),t})  | \leq \frac{3L}{2} s(p_{{\cal R}(a),t}) ,    \notag
\end{align}
for all $a \in {\cal A}$, $j \in {\cal D}_{-{\cal R}(a)}$, $t \in \tau(T)$. Since $|\mu(a, x_{{\cal R}(a),t}) - \pi(a, p_{{\cal R}(a),t})  | \leq L s(p_{{\cal R}(a),t}) /2 $, we have
\begin{align}\tag{A.1}
& |\bar{r}^{({\cal R}(a),j)}_t((p_{{\cal R}(a),t}, p_{j,t}) , a) - \mu(a, x_{{\cal R}(a),t})  | \leq 2 L s(p_{{\cal R}(a),t}),       
\end{align}
for all $a \in {\cal A}$, $j \in {\cal D}_{-{\cal R}(a)}$, $t \in \tau(T)$.
Consider $\hat{c}_t(a)$. Since it is chosen from $\textrm{Rel}_t(a)$ as the type with the minimum variation, we have on the event $\textrm{CORR}_T$
\begin{align}
& |\bar{r}^{(\hat{c}_t(a),k)}_t ((p_{\hat{c}_t(a),t}, p_{k,t}), a) - \bar{r}^{(\hat{c}_t(a),j)}_t ((p_{\hat{c}_t(a),t}, p_{j,t}), a)      | \notag \\
&  \leq 3 L s(p_{{\cal R}(a),t}) , \notag
\end{align}
for all $j,k \in {\cal D}_{-\hat{c}_t(a)}$. 
Hence we have
\begin{align}
& | \bar{r}^{{\cal R}(a)}_t (p_{{\cal R}(a),t},a) - \bar{r}^{\hat{c}_t(a)}_t (p_{\hat{c}_t(a),t},a)  |   \notag \\
& \leq \max_{k,j} \left\{ |\bar{r}^{({\cal R}(a), k)}_t ((p_{{\cal R}(a),t}, p_{k,t}), a) \right. \notag \\
& \left. - \bar{r}^{(\hat{c}_t(a),j)}_t ((p_{\hat{c}_t(a),t}, p_{j,t}) ,a) |  \right\}         \notag \\
& \leq \max_{k,j} \left\{ |\bar{r}^{({\cal R}(a),k)}_t ((p_{{\cal R}(a),t}, p_{k,t}), a)  \right. \notag \\
& \left. - \bar{r}^{({\cal R}(a),\hat{c}_t(a))}_t ((p_{{\cal R}(a),t}, p_{\hat{c}_t(a),t}), a)  | \right. \notag \\
&\left. +  |\bar{r}^{(\hat{c}_t(a),{\cal R}(a))}_t ((p_{\hat{c}_t(a),t}, p_{{\cal R}(a),t}), a)  \right. \notag \\
&\left. - \bar{r}^{(\hat{c}_t(a), j)}_t ((p_{\hat{c}_t(a),t}, p_{j,t}), a)  | \right\} \notag \\
&\leq 6 L s(p_{{\cal R}(a),t}). \tag{A.2}
\end{align}
Combining (A.1) and (A.2), we get
\begin{align}
|\bar{r}^{\hat{c}_t(a)}_t (p_{\hat{c}_t(a),t}, a) - \mu(a, x_{{\cal R}(a),t}) |  \leq 8L s(p_{{\cal R}(a),t}) .  \notag
\end{align}
\end{proof}

Since for $t \in \tau(T)$, $\alpha_t = \argmax_{a \in {\cal A}} \bar{r}^{\hat{c}_t(a)}_t (p_{\hat{c}_t(a),t}, a) $,
using the result of Lemma \ref{lemma:estimatedrel}, we conclude that
\begin{align}
 & \mu_t(\alpha_t) \notag \\
  &\geq \mu_t(a^*(\boldsymbol{x}_t)) - 8 L ( s(p_{{\cal R}(\alpha_t),t}) + s(p_{{\cal R}(a^*(\boldsymbol{x}_t)),t} )  )    ,
\end{align}
Thus, the regret in exploitation steps is
\begin{align}
& 8 L \sum_{t \in \tau(T)} \left( s(p_{{\cal R}(\alpha_t),t}) + s(p_{{\cal R}(a^*(\boldsymbol{x}_t)),t}  )   \right) \notag \\
 &\leq 16 L \sum_{t \in \tau(T)}  \max_{a \in {\cal A}} s(p_{{\cal R}(a),t})  
\leq 16 L \sum_{t \in \tau(T)}  \sum_{i \in {\cal D}}  s(p_{i,t}) \notag \\
&\leq 16 L  D \max_{i \in {\cal D}} \left(   \sum_{t \in \tau(T)} s(p_{i,t}) \right) . \notag 
\end{align}
We know that as time goes on RELEAF uses partitions with smaller and smaller intervals, which reduces the regret in exploitations.
In order to bound the regret in exploitations for any sequence of context arrivals, we assume a worst case scenario, where context vectors arrive such that at each $t$, the active interval that contains the context of each type has the maximum possible length.  
This happens when for each type $i$ contexts arrive in a way that all level $l$ intervals are split to level $l+1$ intervals, before any arrivals to these level $l+1$ intervals happen, for all $l =0,1,2, \ldots$. This way it is guaranteed that the length of the interval that contains the context for each $t \in \tau(T)$ is maximized. 
Let $l_{\max}$ be the level of the maximum level interval in ${\cal P}_i(T)$. For the worst case context arrivals we must have
\begin{align}
& \sum_{l=0}^{l_{\max}-1} 2^l 2^{\rho l} < T     
 \Rightarrow l_{\max} < 1 +\log_2 T/(1+\rho), \notag
\end{align}
since otherwise maximum level hypercube will have level larger than $l_{\max}$. 
Hence we have
\begin{align}
& 16 L  D \max_{i \in {\cal D}} \left(   \sum_{t \in \tau(T)} s(p_{i,t}) \right) 
\leq 16 L D \sum_{l=0}^{1 +\log_2 T/(1+\rho)} 2^l 2^{\rho l} 2^{-l}   \notag \\
&= 16 L D \sum_{l=0}^{1 +\log_2 T/(1+\rho)}  2^{\rho l} 
\leq  16L  D 2^{2 \rho} T^{\rho/(1+\rho)}.  
\end{align}

\section{Proof of Theorem 2} \label{app:thm2}

Recall that time $t$ is an exploitation step only if ${\cal U}_t = \emptyset$.
In order for this to happen we need $S^{\boldsymbol{v(\boldsymbol{q})}}_{t}( \boldsymbol{q},a) \geq D_{i,t}$ for all $\boldsymbol{q} \in Q_i(t)$.
There are $D (D-1)$ type pairs. Whenever action $a$ is explored, all the counters for these $ D (D-1)$ type pairs are updated for the pairs of intervals that contain types of contexts present at time $t$, i.e. $\boldsymbol{q} \in Q_t$. 
Now consider a hypothetical scenario in which instead of updating the counters of all $\boldsymbol{q} \in Q_t$, the counter of only one of the randomly selected interval pair is updated. 
Clearly, the exploration regret of this hypothetical scenario upper bounds the exploration regret of the original scenario. 
In this scenario for any $p_i \in {\cal P}_{i,t}$, $p_j \in {\cal P}_{j,t}$, we have 
\begin{align}
S^{(i,j)}_{t}((p_i, p_j),a )  \leq \frac{ 2 \log(t A D/\delta) }{ L^2 \min(s(p_i), s(p_j))^2 } + 1.  \label{eqn:maxregret}
\end{align}

We can go one step further and consider a second hypothetical scenario where there is only two types $i$ and $j$, for which the actual regret at every exploration step is magnified (multiplied) by $D (D-1)$. 
The maximum possible exploration regret of the second scenario (for the worst case of type $i$ and $j$ context arrivals) upper bounds the exploration regret of the first scenario.
Hence, we bound the regret of the second scenario. 
Let $l_{\max}$ be the maximum possible level for an active interval for type $i$ by time $T$. We must have
$\sum_{l=0}^{l_{\max}-1} 2^{\rho l} < T$,
which implies that $l_{\max} < 1 + \log_2 T/ \rho$.
Next, we consider all pairs of intervals for which the minimum interval has level $l$. 
For each type $j$ interval $p_j$ that has level $l$, there exists no more than $\sum_{k=l}^{l_{\max}} 2^{k} $ type $i$ intervals that have lengths greater than or equal to $l$. Consider a level $k$ type $i$ interval $p_i$ such that $l \leq k < 1 + \log_2 T/ \rho$. Then for the pair of intervals $(p_i, p_j)$ the exploration regret is bounded by $(c_O +1) \left( 2 \log(tAD/\delta)/(2^{-2k} L^2)  +1 \right)$. 
Hence, the worst case exploration regret is bounded by
\begin{align*}
& R_{\textrm{O}}(T) \leq (c_O+1) D^2 \left( 2 \sum_{l=0}^{1 + \log_2 T/ \rho} 2^l \sum_{k=l}^{1 + \log_2 T/ \rho} \right. \notag \\
& \left. 
2^k \left(  \frac{ 2 \log(tAD/\delta)}{ 2^{-2k} L^2}  +1 \right) \right) \\
&= (c_O+1) D^2 \left( \frac{4 \log(tAD/\delta)}{L^2} \sum_{l=0}^{1 + \log_2 T/ \rho} 2^l \sum_{k=l}^{1 + \log_2 T/ \rho} 2^{3k} \right. \\ \notag
&\left.  + 2  \sum_{l=0}^{1 + \log_2 T/ \rho} 2^l \sum_{k=l}^{1 + \log_2 T/ \rho} 2^{k} \right) \\
&\leq \frac{4 D^2 (c_O+1) \log(tAD/\delta)}{L^2} \times \frac{240}{7} T^{4/\rho} \notag \\
&+ \frac{64 D^2 (c_O+1)}{3} T^{2/\rho} . 
\end{align*}

\section{Proof of Theorem 4} \label{app:thm4}

To achieve $\epsilon$-optimality in every exploitation step it is sufficient to have
\begin{align}
& \textrm{INACC}_t(a,j)^C 
= \left\{  |\bar{r}^{({\cal R}(a),j)}_t ((p_{{\cal R}(a),t}, p_{j,t}), a) - \pi(a, p_{{\cal R}(a),t})  | \right . \notag \\
&\left. < \frac{3}{2} L s(p_{{\cal R}(a),t})    \right\} ,   \notag \\
&\subset \left\{  |\bar{r}^{({\cal R}(a),j)}_t ((p_{{\cal R}(a),t}, p_{j,t}), a) - \pi(a, p_{{\cal R}(a),t})  | < \epsilon  \right\} , \notag
\end{align}
for $t \in \tau(T)$. This is satisfied when $l_{\min} \geq \log_2(3L/(2\epsilon))$. Starting with level $l_{\min}$ intervals instead of level $0$ intervals decreases the exploitation regret of ORL-CF. Hence the regret bound in Theorem \ref{thm:exploitbound} is an upper bound on the exploitation regret.

For any sequence of context arrivals, we have the following bound on the level of the interval with the maximum level,
\begin{align}
 l_{\max} < 1 + l_{\min} + \log_2 T/\rho.      \notag
\end{align}
Continuing similarly with the proof of Theorem \ref{thm:exploreregret}, we have
\begin{align*}
& R_{\textrm{O}}(T) \leq (c_O+1) D^2 \left( 2 \sum_{l=0}^{1 + \log_2 T/ \rho} 2^{l_{\min}} 2^l \sum_{k=l}^{1 + \log_2 T/ \rho}
2^{l_{\min}} 2^k \right.  \\
&\left. \left( 2^{4 l_{\min}}  \frac{ 2 \log(tAD/\delta)}{ 2^{-2l_{\min}} 2^{-2k} L^2}  +1 \right) \right) \\
&= (c_O+1) D^2 \left( \frac{4 \log(tAD/\delta)}{L^2} \sum_{l=0}^{1 + \log_2 T/ \rho} 2^l \sum_{k=l}^{1 + \log_2 T/ \rho} 2^{3k} \right. \\
& \left. + 2^{2l_{\min}} 2  \sum_{l=0}^{1 + \log_2 T/ \rho} 2^l \sum_{k=l}^{1 + \log_2 T/ \rho} 2^{k} \right) \\
&\leq 2^{4 l_{\min}} \left(\frac{  4 D^2 (c_O+1) \log(tAD/\delta)}{L^2} \times \frac{240}{7} T^{4/\rho} \right.  \\
& \left. + \frac{64 D^2 (c_O+1)}{3} T^{2/\rho} \right). 
\end{align*}
\jremove{
\section{Proof of Theorem 5} \label{app:thm5}

\subsection{Preliminaries}

Let $A := |{\cal A}|$. We first define a sequence of events which will be used in the analysis of the regret of RELEAF.
For $\boldsymbol{p} \in \boldsymbol{{\cal P}}_{{\cal R}(a),t}$, let $\pi(a, \boldsymbol{p})  = \mu(a, \boldsymbol{x}^*_{{\cal R}(a)}(\boldsymbol{p}) )$, where $\boldsymbol{x}^*_{{\cal R}(a)}(\boldsymbol{p})= \{ x^*_i(p_i) \}_{i \in {\cal R}(a)}$ such that $x^*_i(p_i)$ is the type $i$ context at the geometric center of $p$.
Let $W({\cal R}(a))$ be the set of $D_{\textrm{rel}}$-tuple of types such that ${\cal R}(a) \subset \boldsymbol{w}$, for every $\boldsymbol{w} \in W({\cal R}(a))$. We have 
\begin{align}
|W({\cal R}(a))| = { D- |{\cal R}(a)| \choose 2 D_{\textrm{rel}} - |{\cal R}(a)|  }.      \notag
\end{align}

For a $D_{\textrm{rel}}$-tuple of types $\boldsymbol{w}$, let  
${\cal D}(\boldsymbol{w}, D')$ be the set of $D'$-tuple of types whose elements are from the set ${\cal D}_{-\boldsymbol{w}}$.

For any $\boldsymbol{w} \in W({\cal R}(a))$ and $\boldsymbol{j} \in {\cal D}(\boldsymbol{w}, D_{\textrm{rel}})$, let
\begin{align}
\textrm{INACC}_t(a,\boldsymbol{w},\boldsymbol{j})  
& := \left\{  |\bar{r}^{(\boldsymbol{w},\boldsymbol{j})}_t ( \boldsymbol{p}_{\boldsymbol{w},t}, \boldsymbol{p}_{\boldsymbol{j},t}, a) - 
\pi(a, \boldsymbol{p}_{{\cal R}(a),t})  |  > \right. \notag \\
&\left. \frac{3}{2} L \sqrt{ D_{\textrm{rel}} } 
\max_{i \in {\cal R}(a)} s(\boldsymbol{p}_{{\cal R}(a),t})    \right\} ,   \notag
\end{align}
be the event that the sample mean reward of action $a$ corresponding to the $2 D_{\textrm{rel}}$-tuple of types 
$(\boldsymbol{w}, \boldsymbol{j})$  is {\em inaccurate} for action $a$. 
Let 
\begin{align}
\textrm{ACC}_t(a) := \bigcap_{\boldsymbol{w} \in W({\cal R}(a))}  \bigcap_{\boldsymbol{j} \in {\cal D}(\boldsymbol{w}, D_{\textrm{rel}})}  \textrm{INACC}_t(a,\boldsymbol{w}, \boldsymbol{j})^C \notag
\end{align}
be the event that sample mean reward estimates of action $a$ corresponding to all tuples $(\boldsymbol{w} ,\boldsymbol{j})$ $\boldsymbol{w} \in W({\cal R}(a))$ and $\boldsymbol{j} \in {\cal D}(\boldsymbol{w}, D_{\textrm{rel}})$ are accurate. 
Consider $t \in \tau(T)$. 
Let
\begin{align}
\textrm{WNG}_t(a) := \bigcup_{\boldsymbol{w} \in W({\cal R}(a))} \left\{  \boldsymbol{w} \notin   \textrm{Rel}_t(a)   \right\} \notag
\end{align}
be the event that some $D_{\textrm{rel}}$-tuple that contains ${\cal R}(a)$ is not in the set of relevant tuples of types for action $a$.
Let $\textrm{WNG}_t := \bigcup_{a \in {\cal A}} \textrm{WNG}_t(a)$,
and
$\textrm{CORR}_T := \bigcap_{t \in \tau(T)}   \textrm{WNG}_t^C$,    
be the event that all $D_{\textrm{rel}}$-tuples of types that contain the set of relevant contexts of each action is an element of the set of candidate relevant $D_{\textrm{rel}}$-tuples types corresponding to that action at all exploitation steps. 

We first prove several lemmas related to Theorem 5. The next lemma gives a lower bound on the probability of $\textrm{CORR}_T$.

\begin{lemma}  \label{lemma:corrhigh}
For RELEAF,
for all $a \in {\cal A}$, $t \in \tau(T)$, 
we have 
$\mathrm{P}(\textrm{INACC}_t(a,\boldsymbol{w}, \boldsymbol{j}) )  
\leq \frac{2\delta}{AD^* t^{4}}$. 
for all $\boldsymbol{w} \in W({\cal R}(a))$, $\boldsymbol{j} \in {\cal D}(\boldsymbol{w}, D_{\textrm{rel}})$, and
$\mathrm{P} (\textrm{CORR}_T) \geq 1-\delta$ for any $T$.
\end{lemma}
\begin{proof}
For $t \in \tau(T)$, we have ${\cal U}_t = \emptyset$, hence
\begin{align}
S^{\boldsymbol{v}(\boldsymbol{q})}_t(\boldsymbol{q},a) \geq \frac{2 \log(t A D^* / \delta)}
{ (L \min_{i \in \boldsymbol{v}(\boldsymbol{q})} s(p_{i,t}))^2 } ,      \notag
\end{align}
for all $a \in {\cal A}$, $\boldsymbol{q} \in Q(t)$. 
Due to the {\em Similarity Assumption}, since for all $a \in {\cal A}$, $\boldsymbol{w} \in W({\cal R}(a))$ and $\boldsymbol{j} \in {\cal D}(\boldsymbol{w}, D_{\textrm{rel}})$ the rewards in 
$\bar{r}^{(\boldsymbol{w},\boldsymbol{j})}_t ( (\boldsymbol{p}_{\boldsymbol{w},t}, \boldsymbol{p}_{\boldsymbol{j},t}), a)$
 are sampled from distributions with mean between 
$[\pi(a, \boldsymbol{p}_{{\cal R}(a),t}) - \frac{L\sqrt{D_{\textrm{rel}}}  }{2} 
\max_{ i \in {\cal R}(a) } s(p_{i,t})  , 
\pi(a, \boldsymbol{p}_{{\cal R}(a) , t}) + \frac{L \sqrt{D_{\textrm{rel}}}   }{2}
\max_{ i \in {\cal R}(a) }   s(p_{i,t})  ]$, using a Chernoff bound we get
\begin{align}
& \mathrm{P}(\textrm{INACC}_t(a,\boldsymbol{w},\boldsymbol{j}) ) \notag \\
& \leq 2 \exp\left( 
 -2 (L  \sqrt{D_{\textrm{rel}}} \max_{i \in {\cal R}(a) }s(p_{i,t})  )^2 
 \frac{2 \log(tAD^*/\delta)}{(L \min_{i \in (\boldsymbol{w},\boldsymbol{j})} s(p_{i,t}) )^2 }        \right)   \notag \\
& \leq 2\delta / (AD^* t^{4}) .  \notag
\end{align}
We have
\begin{align}
\textrm{WNG}_t(a) \subset
\bigcup_{\boldsymbol{w} \in W({\cal R}(a))}  \bigcup_{\boldsymbol{j} \in {\cal D}(\boldsymbol{w}, D_{\textrm{rel}})}  \textrm{INACC}_t(a)^C  .   \notag
\end{align}
Since the number of $2D_{\textrm{rel}}$-tuples that contain ${\cal R}(a)$ is 
${D - {\cal R}(a) \choose 2D_{\textrm{rel}} - {\cal R}(a) }$, which is less than or equal to 
$D^* = {D-1 \choose 2D_{\textrm{rel}}-1}$ since $1 \leq {\cal R}(a) \leq D_{\textrm{rel}}$, we have
\begin{align}
\mathrm{P}(   \textrm{WNG}_t(a) ) \leq  2\delta/(A t^{4}),
\end{align}
and
\begin{align}
 \mathrm{P}(   \textrm{WNG}_t ) \leq  2\delta/ t^{4}   .   \notag
\end{align}
This implies that 
\begin{align}
\mathrm{P} (\textrm{CORR}_T^C) &\leq \sum_{t \in \tau(T)}  \mathrm{P}( \textrm{WNG}_t )   \\
&\leq \sum_{t \in \tau(T)} \frac{2\delta}{ t^{4}} \leq \sum_{t=3}^\infty \frac{2\delta}{ t^{4}}  \leq \delta. \notag
\end{align}
\end{proof}

\begin{lemma} \label{lemma:estimatedrelhigh}
When $\textrm{CORR}_T$ happens we have for all $t \in \tau(T)$
\begin{align}
& |\bar{r}^{\hat{c}_t(a)}_t (\boldsymbol{p}_{\hat{c}_t(a),t}, a) - \mu(a, \boldsymbol{x}_{{\cal R}(a),t}) |  \notag \\
&\leq 3 L \sqrt{D_{\textrm{rel}}} ( \max_{ i \in \hat{c}_t(a) }   s(p_{i,t}) + \max_{ i \in \boldsymbol{w} }   s(p_{i,t}) )  \notag \\
&+ 2 L \sqrt{D_{\textrm{rel}}}  \max_{ i \in {\cal R}(a) }   s(p_{i,t})  .  \notag
\end{align}
\end{lemma}
\begin{proof}
From Lemma \ref{lemma:corrhigh}, $\textrm{CORR}_T$ happens when
\begin{align}
& |\bar{r}^{\boldsymbol{w}}_t(\boldsymbol{p}_{\boldsymbol{w},t}, a) - \pi(a, \boldsymbol{p}_{{\cal R}(a),t})| 
 \leq \frac{3L  \sqrt{D_{\textrm{rel}}}  }{2} \max_{ i \in {\cal R}(a) }   s(p_{i,t}) ,    \notag
\end{align}
for all $a \in {\cal A}$, $\boldsymbol{w} \in W({\cal R}(a))$, $t \in \tau(T)$. Since 
\begin{align}
|\mu(a, \boldsymbol{x}_{{\cal R}(a),t}) - \pi(a, \boldsymbol{p}_{{\cal R}(a),t})  | \leq 
\frac{L  \sqrt{D_{\textrm{rel}}}   } {2} \max_{ i \in {\cal R}(a) }   s(p_{i,t})
\end{align}
by the Similarity Assumption, we have
\begin{align}\tag{A.1}
& \hspace{-0.1in} |\bar{r}^{\boldsymbol{w}}_t(  \boldsymbol{p}_{\boldsymbol{w},t}, a) - \mu(a, \boldsymbol{x}_{{\cal R}(a),t})  | 
 \leq 2 L \sqrt{D_{\textrm{rel}}}  \max_{ i \in {\cal R}(a) }   s(p_{i,t}) , \notag        
\end{align}
for all $a \in {\cal A}$, $\boldsymbol{w} \in W({\cal R}(a))$, $t \in \tau(T)$.
Consider $\hat{c}_t(a)$. Since it is chosen from $\textrm{Rel}_t(a)$ as the $D_{\textrm{rel}}$-tuple of types with the minimum variation, we have on the event $\textrm{CORR}_T$
\begin{align}
& |\bar{r}^{(\hat{c}_t(a),\boldsymbol{k})}_t (( \boldsymbol{p}_{\hat{c}_t(a),t}, p_{\boldsymbol{k},t}), a) 
- \bar{r}^{(\hat{c}_t(a),\boldsymbol{j})}_t ((\boldsymbol{p}_{\hat{c}_t(a),t}, p_{\boldsymbol{j},t}), a)      | 
  \leq \notag \\
& 3 L \sqrt{D_{\textrm{rel}}}  \max_{ i \in \hat{c}_t(a) }   s(p_{i,t}) , \notag
\end{align}
for all $\boldsymbol{j},\boldsymbol{k} \in {\cal D}(\hat{c}_t(a), D_{\textrm{rel}})$. 
For any $\boldsymbol{w} \in W({\cal R}(a))$, let $\boldsymbol{g}(\boldsymbol{w}, \hat{c}_t(a))$ be a $2 D_{\textrm{rel}}$-tuple such that for all $i \in \boldsymbol{w}$ and $j \in \hat{c}_t(a)$, $i,j \in \boldsymbol{g}(\boldsymbol{w}, \hat{c}_t(a))$. The existence of at least one such $2 D_{\textrm{rel}}$-tuple of types is guaranteed since $\boldsymbol{w}$ and $\hat{c}_t(a)$ are both $D_{\textrm{rel}}$-tuples of types. Hence, we have for any $\boldsymbol{w} \in W({\cal R}(a))$
\begin{align}
& | \bar{r}^{\boldsymbol{w}}_t (\boldsymbol{p}_{\boldsymbol{w},t},a) - \bar{r}^{\hat{c}_t(a)}_t (\boldsymbol{p}_{\hat{c}_t(a),t},a)  |   \notag \\
& \leq \max_{\boldsymbol{k} \in {\cal D}(\boldsymbol{w}, D_{\textrm{rel}}),\boldsymbol{j} \in {\cal D}(\hat{c}_t(a), D_{\textrm{rel}})} 
\left\{ |\bar{r}^{(\boldsymbol{w}, \boldsymbol{k})}_t ((\boldsymbol{p}_{\boldsymbol{w},t}, p_{\boldsymbol{k},t}), a) 
\right. \notag \\
& \left. - \bar{r}^{(\hat{c}_t(a),\boldsymbol{j})}_t ((\boldsymbol{p}_{\hat{c}_t(a),t}, p_{\boldsymbol{j},t}) ,a) |  \right\}         \notag \\
& \leq \max_{\boldsymbol{k} \in {\cal D}(\boldsymbol{w}, D_{\textrm{rel}}),\boldsymbol{j} \in {\cal D}(\hat{c}_t(a), D_{\textrm{rel}})} 
\left\{ |\bar{r}^{(\boldsymbol{w},\boldsymbol{k})}_t ((\boldsymbol{p}_{\boldsymbol{w},t}, \boldsymbol{p}_{\boldsymbol{k},t}), a)  \right. \notag \\
& \left.  - \bar{r}^{\boldsymbol{g}(\boldsymbol{w}, \hat{c}_t(a))}_t (\boldsymbol{p}_{\boldsymbol{g}(\boldsymbol{w}, \hat{c}_t(a)),t}, a)  | \right. \notag \\
&\left. +  |\bar{r}^{\boldsymbol{g}(\boldsymbol{w}, \hat{c}_t(a))}_t (\boldsymbol{p}_{\boldsymbol{g}(\boldsymbol{w}, \hat{c}_t(a)),t}, a)  
- \bar{r}^{(\hat{c}_t(a), \boldsymbol{j})}_t ((\boldsymbol{p}_{\hat{c}_t(a),t}, \boldsymbol{p}_{\boldsymbol{j},t}), a)  | \right\} \notag \\
&\leq 3 L \sqrt{D_{\textrm{rel}}} ( \max_{ i \in \hat{c}_t(a) }   s(p_{i,t}) + \max_{ i \in \boldsymbol{w} }   s(p_{i,t}) ) . \tag{A.2}
\end{align}
Combining (A.1) and (A.2), we get
\begin{align}
& |\bar{r}^{\hat{c}_t(a)}_t (\boldsymbol{p}_{\hat{c}_t(a),t}, a) - \mu(a, \boldsymbol{x}_{{\cal R}(a),t}) |  \notag \\
&\leq 3 L \sqrt{D_{\textrm{rel}}} ( \max_{ i \in \hat{c}_t(a) }   s(p_{i,t}) + \max_{ i \in {\cal D} }   s(p_{i,t}) )  \notag \\
&+ 2 L \sqrt{D_{\textrm{rel}}}  \max_{ i \in {\cal R}(a) }   s(p_{i,t})  .  \notag
\end{align}
\end{proof}

\subsection{Regret bound for exploitations}

Since for $t \in \tau(T)$, $\alpha_t = \argmax_{a \in {\cal A}} \bar{r}^{\hat{c}_t(a)}_t (p_{\hat{c}_t(a),t}, a) $,
using the result of Lemma \ref{lemma:estimatedrelhigh}, we conclude that
\begin{align}
  \mu_t(\alpha_t)  \geq \mu_t(a^*(\boldsymbol{x}_t)) 
& - 6 L \sqrt{D_{\textrm{rel}}} ( \max_{ i \in \hat{c}_t(a) }   s(p_{i,t}) + \max_{ i \in {\cal D} }   s(p_{i,t}) )  \notag \\
& - 4 L \sqrt{D_{\textrm{rel}}}  \max_{ i \in {\cal R}(a) }   s(p_{i,t})   ,
\end{align}
Thus, the regret in exploitation steps is bounded above by
\begin{align}
&  6 L \sqrt{D_{\textrm{rel}}}   \sum_{t \in \tau(T)}  ( \max_{ i \in \hat{c}_t(a) }   s(p_{i,t}) + \max_{ i \in {\cal D} }   s(p_{i,t}) ) \notag \\
& + 4  L \sqrt{D_{\textrm{rel}}}  \sum_{t \in \tau(T)} \max_{ i \in {\cal R}(a) }   s(p_{i,t}) \notag \\
 &\leq 16 L \sqrt{D_{\textrm{rel}}} \sum_{t \in \tau(T)}  \max_{i \in {\cal D}} s(p_{i,t})  \notag \\
& \leq 16 L \sqrt{D_{\textrm{rel}}}\sum_{t \in \tau(T)}  \sum_{i \in {\cal D}}  s(p_{i,t}) \notag \\
&\leq 16 L  D \sqrt{D_{\textrm{rel}}} \max_{i \in {\cal D}} \left(   \sum_{t \in \tau(T)} s(p_{i,t}) \right) . \notag 
\end{align}
We know that as time goes on RELEAF uses partitions with smaller and smaller intervals, which reduces the regret in exploitations.
In order to bound the regret in exploitations for any sequence of context arrivals, we assume a worst case scenario, where context vectors arrive such that at each $t$, the active interval that contains the context of each type has the maximum possible length.  
This happens when for each type $i$ contexts arrive in a way that all level $l$ intervals are split to level $l+1$ intervals, before any arrivals to these level $l+1$ intervals happen, for all $l =0,1,2, \ldots$. This way it is guaranteed that the length of the interval that contains the context for each $t \in \tau(T)$ is maximized. 
Let $l_{\max}$ be the level of the maximum level interval in ${\cal P}_i(T)$. For the worst case context arrivals we must have
\begin{align}
& \sum_{l=0}^{l_{\max}-1} 2^l 2^{\rho l} < T     
 \Rightarrow l_{\max} < 1 +\log_2 T/(1+\rho), \notag
\end{align}
since otherwise maximum level hypercube will have level larger than $l_{\max}$. 
Hence, we have
\begin{align}
& 16 L  D \sqrt{D_{\textrm{rel}}} \max_{i \in {\cal D}} \left(   \sum_{t \in \tau(T)} s(p_{i,t}) \right) \\ \notag
& \leq 16 L D \sqrt{D_{\textrm{rel}}} \sum_{l=0}^{1 +\log_2 T/(1+\rho)} 2^l 2^{\rho l} 2^{-l}   \notag \\
&= 16 L D \sqrt{D_{\textrm{rel}}} \sum_{l=0}^{1 +\log_2 T/(1+\rho)}  \hspace{-0.1in} 2^{\rho l} \notag \\
& \leq  16L D \sqrt{D_{\textrm{rel}}}   2^{2 \rho} T^{\rho/(1+\rho)}.  
\end{align}

Hence, we have $R_I(T) = \tilde{O}(T^{\rho/(1+\rho)})$ with probability $1-\delta$.

\subsection{Regret bound for explorations}

Recall that time $t$ is an exploitation step only if ${\cal U}_t = \emptyset$.
In order for this to happen we need $S^{\boldsymbol{v(\boldsymbol{q})}}_{t}( \boldsymbol{q},a) \geq D_{i,t}$ for all $\boldsymbol{q} \in Q_i(t)$.
The number of distinct $2 D_{\textrm{rel}}$-tuples of types is ${D \choose 2 D_{\textrm{rel}} }$. Whenever action $a$ is explored, all the counters for these ${D \choose 2 D_{\textrm{rel}} }$ type tuples are updated for the $2 D_{\textrm{rel}}$-tuples  of intervals that contain types of contexts present at time $t$, i.e. $\boldsymbol{q} \in Q_t$. 
Now consider a hypothetical scenario in which instead of updating the counters of all $\boldsymbol{q} \in Q_t$, the counter of only one of the randomly selected $2 D_{\textrm{rel}}$-tuple of intervals is updated. 
Clearly, the exploration regret of this hypothetical scenario upper bounds the exploration regret of the original scenario. 
In this scenario for any $\boldsymbol{q} \in Q_t$, we have 
\begin{align}
S^{\boldsymbol{v}(\boldsymbol{q}) }_{t}(\boldsymbol{q} ,a )  \leq \frac{ 2 \log(t A D^*/\delta) }
{ (L \min_{i \in \boldsymbol{v}(\boldsymbol{q} )} s(p_i) )^2 } + 1.  \label{eqn:maxregrethigh}
\end{align}

We fix a $2 D_{\textrm{rel}}$-tuple of types $\boldsymbol{j} = (j_1, j_2, \ldots, j_{2 D_{\textrm{rel}}})$, and analyze the worst-case regret due to exploration of this tuple of types, which is denoted by $R_{O,\boldsymbol{j}}(T)$. Since there are ${D \choose 2 D_{\textrm{rel}} }$ of such tuples of types, an upper bound on the exploration regret is ${D \choose 2 D_{\textrm{rel}} } R_{O,\boldsymbol{j}}(T)$.
 
Let $l_{\max}$ be the maximum possible level for an active interval for type $i$ by time $T$. We must have
$\sum_{l=0}^{l_{\max}-1} 2^{\rho l} < T$,
which implies that $l_{\max} < 1 + \log_2 T/ \rho$. Let $\gamma= 1 + \log_2 T/ \rho$.

First, we will consider the exploration regret incurred in all configurations where type $j_n$'s  intervals has levels $l_n$, for $n=1,2,\ldots,2 D_{\textrm{rel}}$ such that $l_1 \leq l_2 \leq \ldots \leq l_{2 D_{\textrm{rel}}}$. We denote this ordering by $\boldsymbol{j}^*$ and the exploration regret in this ordering by 
$R_{O,\boldsymbol{j}^*}(T)$. There are $(2 D_{\textrm{rel}})!$ different configurations in which the orderings of levels of the intervals of the types are different. 

Let $z = 2 D_{\textrm{rel}}$. Consider the tuple of intervals $(p_{j^*_1}, \ldots, p_{j^*_{2 D_{\textrm{rel}}}})$. The exploration regret for this tuple of intervals is bounded by 
\begin{align}
(c_O +1) \left( 2 \log(T AD^*/\delta)/(2^{-2 l_{z}} L^2)  +1 \right)  .    \notag
\end{align}
Hence, we have 
\begin{align}
& R_{O,\boldsymbol{j}^*}(T) \leq (c_O+1)  \notag \\
& \times \sum_{l_1=0}^{\gamma} 2^{l_1} \sum_{l_2=l_1}^{\gamma} 2^{l_2} 
\ldots \sum_{l_z = l_{z-1}}^{\gamma} 2^{l_z}  \left( \frac{2 \log(T AD^*/\delta)}{2^{-2 l_{z}} L^2}  +1 \right)   \notag \\
& \leq (c_O+1) \sum_{l_z = l_{z-1}}^{\gamma} 2^{l_z}   \ldots \sum_{l_{z-1} = l_{z-2}}^{\gamma} 2^{l_{z-1} }
   O(  T^{3/\rho}  \log T)    \notag \\
& \leq  (c_O+1) \sum_{l_z = l_{z-1}}^{\gamma} 2^{l_z}   \ldots \sum_{l_{z-2} = l_{z-3}}^{\gamma} 2^{l_{z-2} }
   O(  T^{4/\rho}  \log T) \notag \\
& = O(  T^{(2+2D_{\textrm{rel}})/\rho}  \log T) .
\end{align}

Since $R_O(T) \leq {D \choose 2 D_{\textrm{rel}} } (2 D_{\textrm{rel}})! R_{O,\boldsymbol{j}^*}(T)$, we have 
$R_O(T) = O(  T^{(2+2D_{\textrm{rel}})/\rho}  \log T)$.

\subsection{Balancing the regret due to exploitations and explorations}

From the results of the previous subsections we have with probability $1-\delta$, $R_I(T) = \tilde{O}(T^{\rho/(1+\rho)})$ and $R_O(T) = \tilde{O}(  T^{(2+2D_{\textrm{rel}})/\rho} )$.
Since $R_I(T)$ is increasing in $\rho$ and $R_O(T)$ is decreasing in $\rho$ there is a unique $\rho$ for which they are equal. This unique solution is 
\begin{align}
\rho =   \frac{2+2D_{\textrm{rel}} + \sqrt{ 4 D^2_{\textrm{rel}} + 16 D_{\textrm{rel}} +12 } }  
   { 4+ 2 D_{\textrm{rel}} + \sqrt{ 4 D^2_{\textrm{rel}} + 16 D_{\textrm{rel}} +12 }} .    \notag
\end{align}

}

\end{document}